\documentclass[letterpaper, 10 pt, journal, twoside]{ieeetran}
\usepackage{times}

\usepackage{multicol}

\newtheorem{theorem}{Theorem}[section]

\newtheorem{prop}[theorem]{Proposition}

\newtheorem{problem}{Problem}

\newtheorem{rem}[theorem]{Remark}

\usepackage{hyperref}
\usepackage{cite}

\usepackage{algorithm}
\usepackage[noend]{algpseudocode}
\usepackage{multirow}
\usepackage{color}
\usepackage{amsfonts}
\usepackage{mysymbol}

\usepackage{bbold}
\usepackage{amsmath}

\usepackage{amssymb}

\usepackage{mathtools}
\DeclarePairedDelimiter{\ceil}{\lceil}{\rceil}

\IEEEoverridecommandlockouts                

\begin{document}
\title{\textcolor{black}{Probabilistically Correct Language-based Multi-Robot Planning using Conformal Prediction}}
\author{Jun Wang, Guocheng He, and Yiannis Kantaros
\thanks{The authors are with Department of Electrical and Systems Engineering, Washington University in St Louis, MO, 63108, USA. Emails: $\left\{\text{junw, guocheng,ioannisk}\right\}$@wustl.edu.  
This work was supported by the ARL grant DCIST CRA W911NF-17-2-0181 and the NSF awards CNS $\#2231257$, and CCF $\#2403758$.
}}

\maketitle

\begin{abstract}
This paper addresses task planning problems for language-instructed robot teams. Tasks are expressed in natural language (NL), requiring the robots to apply their capabilities at various locations and semantic objects.  Several recent works have addressed similar planning problems by leveraging pre-trained Large Language Models (LLMs) to design effective multi-robot plans. However, these approaches lack \textcolor{black}{performance} guarantees. To address this challenge, we introduce a new \textcolor{black}{distributed} LLM-based planner, called S-ATLAS for Safe plAnning for Teams of Language-instructed AgentS, that is capable of achieving user-defined mission success rates. This is accomplished by leveraging conformal prediction (CP), a distribution-free uncertainty quantification tool in black-box models. CP allows the proposed multi-robot planner to reason about its inherent uncertainty in a \textcolor{black}{distributed} fashion, enabling robots to make individual decisions when they are sufficiently certain and seek help otherwise. We show, both theoretically and empirically, that the proposed planner can achieve user-specified task success rates, \textcolor{black}{assuming successful plan execution}, while minimizing the overall number of help requests. We provide comparative experiments against related works showing that our method is significantly more computational efficient and achieves lower help rates. The advantage of our algorithm over baselines becomes more pronounced with increasing robot team size.
\end{abstract}

\IEEEpeerreviewmaketitle


\section{Introduction}\label{sec:intro}


\IEEEPARstart{I}{n} recent years, the field of robotics has witnessed a paradigm shift towards the deployment of multiple robots collaborating to accomplish complex missions where effective coordination becomes paramount \cite{schlotfeldt2018anytime,kantaros2016distributed,gosrich2022coverage}. To this end, several multi-robot task planning algorithms have been proposed. These planners, given a high-level mission, can allocate tasks to robots and then design individual sequences of high-level actions to accomplish the assigned tasks \cite{elimelech2022efficient,kantaros2020stylus,turpin2014capt,pianpak2019distributed,fang2022automated,chen2024fast}. 
Execution of these action plans is achieved using low-level motion planners and controllers \cite{vasilopoulos2018reactive,karaman2011sampling,kavraki1996probabilistic}. A comprehensive survey on task and motion planning can be found in \cite{garrett2021integrated,antonyshyn2023multiple}. 
Despite these notable achievements in robot planning, a recurrent limitation in these techniques is the substantial user expertise often required for mission specification using e.g., formal languages \cite{baier2008principles} or reward functions \cite{sutton2018reinforcement}.

\textcolor{black}{Natural Language (NL) has also been explored as a more user-friendly means to specify robot missions. Early research in this field primarily focused on mapping NL to planning primitives. These early works often utilized statistical machine translation \cite{koehn2009statistical} to identify data-driven patterns for translating free-form commands into a formal language defined by a grammar \cite{tellex2011understanding,matuszek2010following,chen2011learning,wong2006learning,kollar2010toward,howard2014natural}. However, these approaches were limited to structured state spaces and simple NL commands. Motivated by the remarkable generalization abilities of pre-trained Large Language Models (LLMs) across diverse task domains \cite{achiam2023gpt,touvron2023llama,driess2023palm,xi2023rise}, there has been increasing attention on utilizing LLMs for NL-based planning. LLMs empower robots to design plans through conversational interactions to handle complex NL instructions.} Early efforts primarily focused on single-robot task planning problems \cite{singh2023progprompt, liang2023code,shah2023lm,xie2023translating,ding2023task,liu2023llm+,wu2023tidybot,zeng2022socratic,stepputtis2020language,li2022pre,huang2022inner,ruan2023tptu,ahn2022can,luo2023obtaining,joublin2023copal,dai2023optimal,yang2024text2reaction,tang2023graspgpt,rana2023sayplan,ravichandran2024spine} while recent extensions to multi-robot systems are presented in \cite{mandi2023roco,zhang2023building,talebirad2023multi,liu2023bolaa,hong2023metagpt,chen2023scalable,zhang2023controlling,chen2024solving,kannan2023smart,liu2024leveraging}. These multi-robot planners either delegate each robot to an LLM for decentralized plan construction, use LLMs as centralized planners, or explore hybrid multi-agent communication architectures. Additionally, mechanisms to detect conflicts, such as collisions in the plans, and provide feedback to the LLMs for plan revision have been integrated into these frameworks. A detailed survey can be found in \cite{pallagani2024prospects,zeng2023large,hunt2024survey}.
A major challenge with current LLM-based planners is that they typically lack mission performance and safety guarantees while they often hallucinate, i.e., they confidently generate incorrect and possibly unsafe outputs.

\begin{figure}[t] 
\centering
\includegraphics[width=\linewidth]{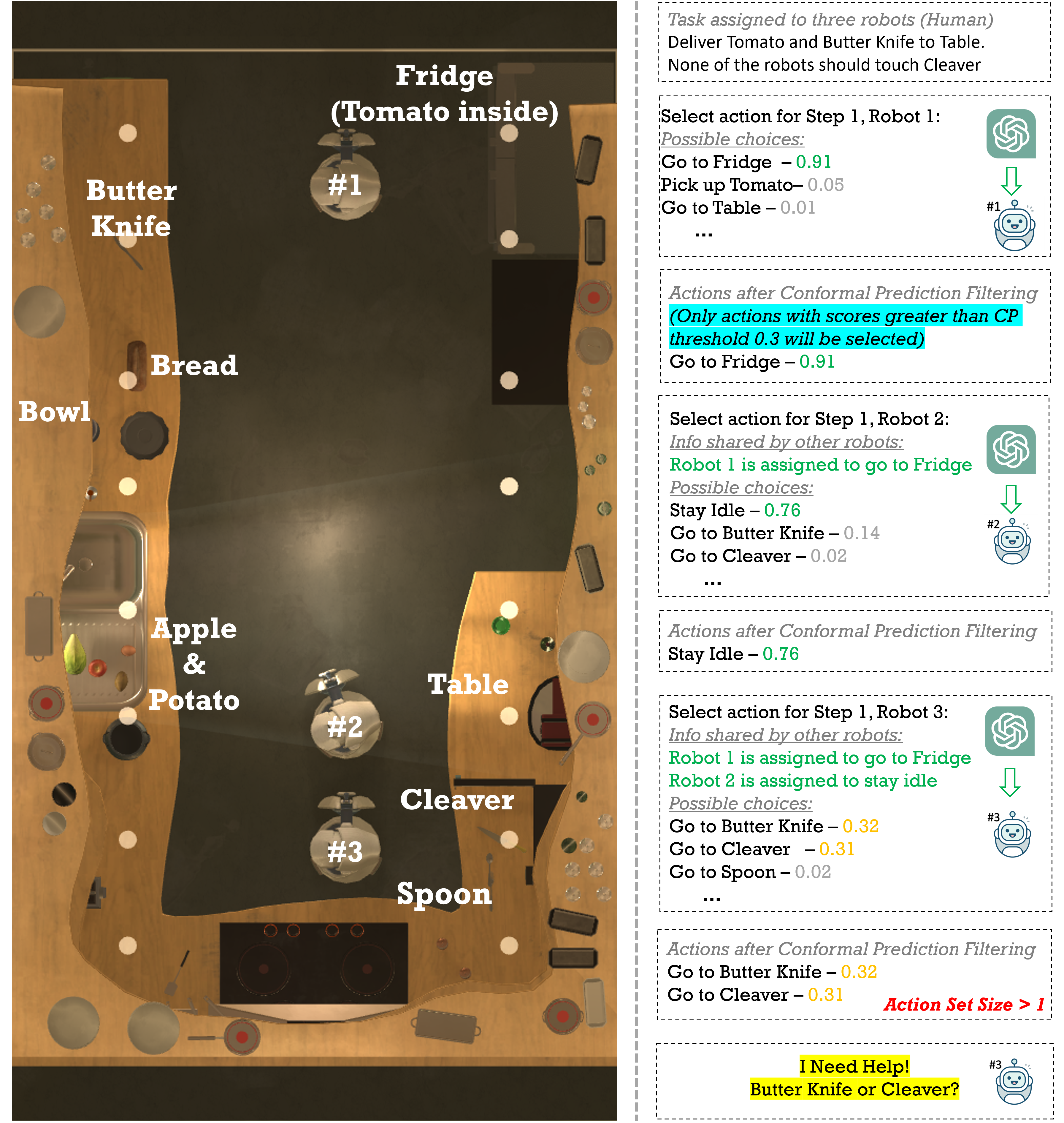}\vspace{-0.3cm}
\caption{ 
We propose a \textcolor{black}{distributed} planner for language-instructed multi-robot systems being capable of achieving user-specified mission success rates. The robots are delegated to pre-trained Large Language Models (LLMs), enabling them to select actions while coordinating in a conversational manner. Our LLM-based planner reasons about its inherent uncertainty, using conformal prediction, in a decentralized fashion. This capability allows the planner to determine when and which robot is uncertain about correctness of its next action. In cases of high uncertainty, the respective robots seek assistance.
%
} 
\label{fig:env}\vspace{-0.5cm}
\end{figure}
This paper focuses on enhancing the reliability of LLM-based multi-robot planners. We consider teams of robots possessing various skills such as mobility, manipulation, and sensing and tasked with high-level mission expressed using NL. Mission accomplishment  requires the robots to apply their skills to various known semantic objects that exist in the environment. 
Our overarching goal is to design language-based planners capable of computing multi-robot plans (defined as sequences of robot actions) \textcolor{black}{with a user-specified probability of correctness; this ensures desired mission success rates assuming successful plan execution.}  This requires developing LLM-based planners that can reason about uncertainties of the employed LLMs, enabling robots to make decisions when sufficiently confident and seek help otherwise. 

%


To address this problem, we propose a new \textcolor{black}{distributed} LLM-based planner, called S-ATLAS for Safe plAnning for Teams of Language-instructed AgentS; see also Fig. \ref{fig:env}.
\textcolor{black}{In our framework, at each time step, the robots select actions sequentially while considering actions chosen by other robots as e.g., in \cite{mandi2023roco,chen2023scalable}.} Each robot is delegated to a pre-trained LLM agent that is responsible for decision making. This coordinate-descent approach enables the \textcolor{black}{distributed} construction of robot plans.
%
%
To choose an action for a robot, we present the action selection problem to the corresponding LLM agent as a multiple-choice question-answering (MCQA) scenario \cite{ahn2022can}. Here, the `question' corresponds to the textual task description and the history of past decisions, while the set of available `choices' represents the skills that the selected robot can apply.
%
The MCQA framework ensures that, unlike in other multi-robot planners, the LLM only chooses from valid choices, mitigating (partially) the risk of hallucination where invalid plans with nonsensical actions may be generated. Instead of simply selecting the action with the highest logit score, we employ conformal prediction (CP) to quantify the uncertainty of the employed LLM \cite{balasubramanian2014conformal,angelopoulos2023conformal,shafer2008tutorial}. \textcolor{black}{This is crucial as LLMs tend to confidently generate incorrect outputs} \cite{kumar2023conformal,ren2023robots}. 
\textcolor{black}{CP enables the design of individual prediction sets for each robot that contain the correct action with high confidence.} \textcolor{black}{This allows each LLM-driven robot to determine when it is uncertain about its decisions, which is a key step toward achieving  desired mission success rates.}
%
In cases of high uncertainty, indicated by non-singleton prediction sets, the respective robots seek assistance from other robots and/or users; \textcolor{black}{otherwise, they execute the action in their singleton sets.}

\textbf{Related Works:} \textcolor{black}{As discussed earlier, several planners for language-driven robot teams have been proposed recently which, unlike the proposed method, lack performance guarantees; see e.g.,  \cite{mandi2023roco,chen2023scalable}.}
The closest works to ours are \cite{kumar2023conformal,wang2023conformal,ren2023robots}, as they apply CP for uncertainty alignment of LLMs. In \cite{kumar2023conformal}, CP is applied to single-step MCQA tasks while \cite{ren2023robots,wang2023conformal} build upon that approach to address more general NL tasks, modeled as multi-step MCQA problems, with desired success rates. Notably, the inspiring work in \cite{ren2023robots} is the first to demonstrate how CP can determine when a robot needs help from a user to efficiently resolve task ambiguities. These works, however, focus on single-robot planning problems. 
\textcolor{black}{We note that \cite{ren2023robots} can be extended to multi-robot settings by treating the multi-robot system as a single robot performing multiple actions simultaneously. However, we show that these centralized implementations do not scale well with increasing team size due to increasing computational costs and/or the diminishing ability of a single LLM to effectively manage large robot teams. Moreover, centralized implementations of \cite{ren2023robots} generate multi-robot/global prediction sets, making it harder to determine which specific robot requires assistance. Our method addresses these challenges through its distributed nature, constructing local prediction sets, for each robot, allowing users to easily identify which robot needs help and when, while also ensuring desired team-wide mission success rates. Comparisons against these centralized approaches demonstrate that our method is significantly more computationally efficient and requires fewer user interventions, particularly as team size and mission complexity grow. } 
Finally, we note that several uncertainty quantification and calibration methods for LLMs have been proposed recently that 
associate uncertainty with point-valued outputs \cite{xiao2019quantifying,zhou2023navigating,xiao2022uncertainty}. Unlike these methods, CP generates prediction sets containing the ground truth output with user-specified probability providing coverage guarantees for the underlying model. \textcolor{black}{CP has also been applied recently for various robotics tasks, such as perception \cite{angelopoulos2020uncertainty,yang2023object,mao2023safe,su2024collaborative}, motion prediction \cite{lindemann2023safe,sun2023conformal,cleaveland2023conformal}, and decision making \cite{lekeufack2023conformal,dixit2024perceive}; a comprehensive review of CP applications and implementations can be found in \cite{manokhin_valery_2022_6467205}. \textcolor{black}{To the best of our knowledge, this paper presents the first CP application in NL-based multi-robot planning, particularly in \textit{distributed} settings.}}


\textcolor{black}{\textbf{Evaluation:} We provide extensive comparisons against recent LLM-based planners \cite{chen2023scalable,ren2023robots}. First, we compare our method against centralized implementations of \cite{ren2023robots} that can also achieve desired mission completion rates using CP as discussed above. Our comparative experiments demonstrate our method is significantly more scalable and computationally efficient and results in lower help rates from users. 
} 
Second, in Appendix \ref{app:Comp}, we compare our method against centralized and decentralized planners proposed in \cite{chen2023scalable}. We note that these planners do not allow robots to ask for help. Our comparative experiments demonstrate  that our method outperforms \cite{chen2023scalable} in terms of its ability to design correct plans, even when the CP-based help-request module is omitted from S-ATLAS. This performance gap between our method and the baselines becomes more evident as the robot team size and the mission complexity increase. 

\textbf{Contribution:} The contribution of the paper can be summarized as follows. \textbf{(i)} We propose the first \textcolor{black}{\textit{distributed}} planner for language-instructed multi-robot systems that can achieve \textit{user-specified task success rates} \textcolor{black}{(assuming successful execution of robot skills)}. \textbf{(ii)} We show how conformal prediction can be applied in a distributed fashion to construct \textcolor{black}{local prediction sets for each individual robot.} 
\textbf{(iii)} \textcolor{black}{We provide  comparative experiments showing that our planner outperforms related language-based planners in terms of efficiency, planning performance, and help rates.}
\vspace{-0.1cm}
\section{Problem Formulation}\label{sec:problem}
\vspace{-0.1cm}


\textbf{Robot Team:} Consider a team of $N\geq 1$ robots. Each robot is governed by the following dynamics: $\bbp_j(t+1)=\bbf_j(\bbp_j(t),\bbu_j(t))$, $j\in\{1, 2, \dots, N\}$, where $\bbp_j(t)$ and $\bbu_j(t)$ stand for the state the control input of robot $j$ at discrete time $t$, respectively. We assume that all robot states are known for all time instants $t\geq0$.
Each robot possesses $A>0$ skills collected in a set $\ccalA\in\{1,\dots,A\}$. These skills $a\in\ccalA$ are represented textually (e.g., `take a picture', `grab', or `go to'). The application of skill $a$ by robot $j$ to an object/region at location $\mathbf{x}$ at time $t$ is represented as $s_j(a, \mathbf{x}, t)$. For simplicity, we assume homogeneity among the robots, meaning they all share the same skill set $\mathcal{A}$ and can  apply skill $a$ at any given location $\mathbf{x}$ associated with an object/region; later we discuss how this assumption can be relaxed. When it is clear from the context, we denote $s_j(a,\bbx,t)$ by $s_j(t)$ for brevity. We also define the multi-robot action at time $t$ as $\bbs(t)=[s_1(t),\dots,s_N(t)]$. The time step $t$ is increased by one, once $\bbs$ is completed/executed.
Also, we assume that every robot has access to low-level controllers to execute the skills in $\ccalA$. We assume error-free execution of these skills \cite{ren2023robots}.\footnote{This assumption will be used in Theorem \ref{thm1}; see also Section \ref{sec:concl}.}.

\textbf{Environment:} The robot team resides in an environment $\Omega \subseteq \mathbb{R}^d, d \in \{2,3\}$, with obstacle-free space $\Omega_{\text{free}} \subseteq \Omega$. We assume that $\Omega_{\text{free}}$ contains $M>0$ semantic objects. \textcolor{black}{Each object $e$ is defined by its location $\mathbf{x}_e$ and a semantic label $o_e$.}
%
%
Using the action space $\ccalA$ and the available semantic objects, we can construct a set $\mathcal{S}$ that collects all possible decisions $s_j$ that a robot can make. Notice that since we consider homogeneous robots, the set $\ccalS$ is the same for all robots $j$.
We assume that both the obstacle-free space and the locations/labels of all objects are a-priori known and static; in Section \ref{sec:planner}, we discuss how this assumption can be relaxed.
%
%

\textbf{Mission Specification:} The team is assigned a high-level coordinated task $\phi$, expressed in natural language (NL), that is defined over \textcolor{black}{the objects or regions of interest in $\Omega$.} 
Note that $\phi$ may comprise multiple sub-tasks that are not necessarily pre-assigned to specific robots. Thus, achieving $\phi$ involves determining which actions $a$ each robot $j$ should apply, when, where, and in what order. Our goal is to design multi-robot mission plans that accomplish $\phi$, defined as $\tau = {\bbs(1), \bbs(2), \dots, \bbs(H)}$, for some horizon $H > 0$.
We formalize this by considering a distribution $\ccalD$ over scenarios $\xi_i=\{N_i, \ccalA_i, \phi_i, H_i, \Omega_i\}$, with corresponding ground truth plans $\tau_i$.\footnote{To account for heterogeneous robots, more complex distributions $\ccalD$ can be defined generating individual action sets. The proposed algorithm remains applicable. We abstain from this presentation for simplicity.} Recall that $N_i, \ccalA_i, \phi_i, H_i, \text{~and~} \Omega_i$ refer to the number of homogeneous robots, the robot skills, the language-based task, the mission horizon, and the semantic environment, respectively, associated with $\xi_i$. The horizon $H_i$ essentially determines an upper bound on the number of steps $t$ required to accomplish $\phi_i$. The subscript $i$ is used to emphasize that these parameters can vary across scenarios. We assume that all scenarios $\xi_i$ drawn from $\ccalD$ are feasible in the sense that $H_i$ is large enough and all objects of interest are accessible to the robots. When it is clear from the context, we drop the dependence on $i$. Note that $\ccalD$ is unknown but we assume that we can sample i.i.d. scenarios from it.
\textcolor{black}{Our goal is to design \textcolor{black}{probabilistically correct} language-based planners, i.e., planners that can generate \textcolor{black}{correct/feasible} plans $\tau_i$ satisfying $\xi_i$, in at least $(1-\alpha)\cdot 100$\% of scenarios $\xi_i$ drawn from $\ccalD$, for a user-specified $\alpha\in (0,1)$.}\footnote{\textcolor{black}{We note that probabilistic task satisfaction arises solely due to imperfections of LLMs employed to process NL. Our future work will consider imperfections of robot skills as well.}} The problem that this paper addresses can be summarized as follows:

\begin{problem}\label{pr:pr1}
    \textcolor{black}{Design a planner that generates multi-robot plans $\tau$, which complete missions $\xi$, drawn from an unknown distribution $\ccalD$, with probability greater than a user-specified threshold $1-\alpha\in(0,1)$, assuming successful execution of $\tau$. }
\end{problem}

\section{\textcolor{black}{Probabilistically Correct Language-based Multi-robot Planning}}\label{sec:planner}

\begin{algorithm}[t]
\caption{S-ATLAS}\label{alg:MultiRobotLLM}
\begin{algorithmic}[1]
\State \textbf{Input}: Scenario $\xi$; Coverage level $\alpha$; Upper bound $W$
\State Initialize an ordered set $\ccalI$;\label{algo1:order_init}
\State Initialize prompt $\ell_{\ccalI(1)}(1)$ with empty history of actions\;\label{algo1:prompt_init} 
\For{$t=1~\text{to}~H$}
\State Initialize help counter $w=0$\;\label{algo1:counter1}
\State $\bbs(t)=[\varnothing,\dots,\varnothing]$\label{algo1:nullAction}\;
\For{$i\in\ccalI$}\label{algo1:for1}
\State Next robot $j=\ccalI(i)$\label{algo1:nextRbt}
\State Compute $\ell_j(t)$ using $\ell_{j-1}(t)$\;\label{algo1:constrPrompt}
\State Compute the local prediction set $\mathcal{C}(\ell_j(t))$ 
as in \eqref{eq:pred3local}\label{algo1:pred_set}
\If{ $|\mathcal{C}(\ell_j(t))|>1$ }\label{algo1:HelpTriggered}
\State $w=w+1$\;\label{algo1:helpCounterIncr}

\If{$w>W$}\label{algo1:HelpHuman1}
\State Obtain $s_j(t)$ from human operator\;  \label{algo1:HelpHuman2}
\Else
\State Get new $\ccalI'$, set $\ccalI=\ccalI'$, and go to line \ref{algo1:nullAction}\;\label{algo1:helpteam}
\EndIf
\Else
\State  Pick the (unique) decision $s_j(t)\in\mathcal{C}(\ell_j(t))$\label{algo1:pickAction}
\EndIf
\State Update $\ell_{j}(t)=\ell_{j}(t)+s_j(t)$ \label{algo1:prompt_update}
\State Send $\ell_j(t)$ to robot $\ccalI(i+1)$\label{algo1:sendInfo}
\EndFor
\State Construct $\bbs(t) = [s_1(t),\dots,s_N(t)]$\label{algo1:s}
\State Append $\bbs(t)$ to the multi-robot plan $\tau$\label{algo1:tau}
\EndFor
\end{algorithmic}
\end{algorithm}
\normalsize

In this section, we propose a new task planner that harnesses pre-trained LLMs to address Problem \ref{pr:pr1}. In Section
\ref{sec:centralized}, \textcolor{black}{we present a centralized planner, as a potential solution to Problem \ref{pr:pr1}, building upon \cite{ahn2022can,ren2023robots}, which, however, cannot effectively scale to multi-robot settings}. To address this challenge, in Sections \ref{sec:coorddescent}-\ref{sec:cp}, we present S-ATLAS, our proposed \textcolor{black}{distributed} planner; see Alg. \ref{alg:MultiRobotLLM} and Fig. \ref{fig:cd_framework}.

\vspace{-0.1cm}
\subsection{A Centralized Planning Framework using LLMs}\label{sec:centralized}
\vspace{-0.1cm}

\textcolor{black}{In this section, building on the single-robot planners from \cite{ahn2022can,ren2023robots}, we present a centralized multi-robot planner that can address Problem \ref{pr:pr1}. In \cite{ahn2022can}, the task planning problem is modeled as a sequence of MCQA scenarios over time steps $t \in \{1,\dots,H\}$, where the question represents the task $\phi$ and mission progress, and the choices represent decisions $s_j(t) \in \ccalS$ that (the single) robot $j$ can make. This approach can be extended to multi-robot systems by treating the team as a single robot executing $N$ actions simultaneously, leading to $S^N$ choices in the MCQA framework. Based on \cite{ahn2022can}, the action with the highest LLM confidence score is selected; assuming perfect robot skills (i.e., affordance functions that always return 1). However, these scores are uncertainty-agnostic and uncalibrated. Therefore, this approach cannot effectively solve Problem \ref{pr:pr1} \cite{ren2023robots}. Instead, we apply the conformal prediction (CP) approach from \cite{ren2023robots} allowing the team to make decisions when sufficiently certain and seek help from users otherwise. As shown in \cite{ren2023robots}, this centralized approach can achieve desired mission success rates. However, application of CP in this setting becomes computationally intractable as $N$ increases due to the exponential growth of the number $S^N$ of options that need to be considered; see Section \ref{sec:knownoComp}.}

\begin{rem}[Enhancing Computational Efficiency]\label{rem:baseline2}
    \textcolor{black}{A more efficient centralized approach would be to consider only $X\ll S^N$ options (i.e., multi-robot decisions) in every MCQA scenario, that are generated on-the-fly by an LLM as proposed in \cite{ren2023robots}. As in \cite{ren2023robots}, these options should be augmented with an additional choice `other option, not listed' to account for cases where all LLM options are invalid. A challenge in this case is that as $N$ increases, LLMs struggle to provide valid options, often leading to `incomplete' plans ending with 'other option, not listed'; see Section \ref{sec:knownoComp}.}
\end{rem}

\begin{figure}[t] 
\centering
\includegraphics[width=1\linewidth]{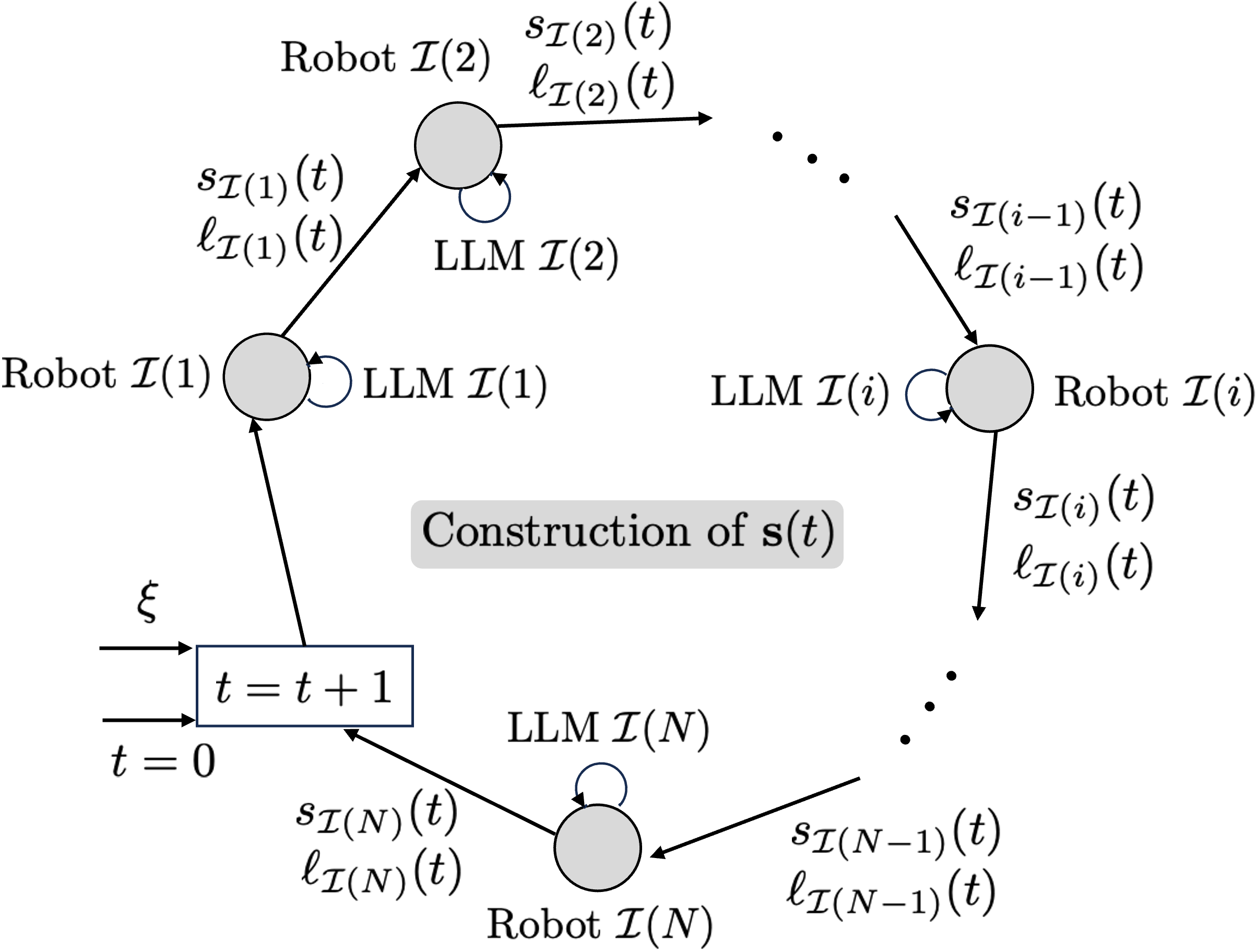} \vspace{-0.5cm}
\caption{Given a task scenario $\xi$, the robots select decisions at each time $t$ sequentially as per an ordered set $\ccalI$ of robot indices. Each robot $j=\ccalI(i)$ (gray disk) is delegated to an LLM that generates a decision $s_j(t)$ given textual context $\ell_{\ccalI(i-1)}(t)$ containing the task description and past robot decisions provided by the previous robot $\ccalI(i-1)$. If the LLM for robot $j$ is not certain enough about what the correct decision $s_j(t)$ is, the robots seek assistance (not shown). } \vspace{-0.5cm}
\label{fig:cd_framework}
\end{figure}

\vspace{-0.1cm}
\subsection{\textcolor{black}{Distributed} Multi-Robot Planning using LLMs}\label{sec:coorddescent}
\vspace{-0.1cm}
In this section, we propose S-ATLAS, \textcolor{black}{a distributed planning framework to address Problem \ref{pr:pr1} and overcome the limitations of the centralized approaches discussed in Section \ref{sec:centralized}, making it more effective for larger-scale planning problems.} 
The key idea is that, given a scenario $\xi$ drawn from $\ccalD$, the robots pick decisions at time $t$ successively (as opposed to jointly in Section \ref{sec:centralized}). Each robot $j$ is delegated to an LLM that picks a decision $s_j(t)$, using an MCQA setup, while incorporating decisions of past robot decisions. If the LLM for robot $j$ is not certain enough about what the correct $s_j(t)$ is, it seeks assistance from other robots and/or human operators. This process gives rise to the multi-robot decision $\bbs(t)$ at time $t$. 
The above is repeated to construct $\bbs(t+1),\dots,\bbs(H)$. Next, we describe our proposed \textcolor{black}{distributed} planning algorithm in more detail.

Consider a scenario $\xi$ drawn from $\ccalD$ and an ordered set $\ccalI$, initialized by a user, collecting robot indices $j\in\{1,\dots,N\}$ so that each robot index appears exactly once; e.g., $\ccalI=\{1,\dots,N\}$ [line \ref{algo1:order_init}, Alg. \ref{alg:MultiRobotLLM}]. We denote the $i-$th element in $\ccalI$ by $\ccalI(i)$. At each $t\in \{1\dots H \}$, the robots pick actions sequentially as per $\ccalI$ [lines \ref{algo1:prompt_init}-\ref{algo1:tau}, Alg. \ref{alg:MultiRobotLLM}]; see also Fig. \ref{fig:cd_framework}. Specifically, robot $j=\ccalI(i)$ selects $s_j(t)$ while considering, the NL task description as well as all decisions that have been made until $t-1$ and all decisions made by the robots $\ccalI(1),\dots,\ccalI(i-1)$ at time $t$. This information, denoted by $\ell_j(t)$, is described textually. Given the context $\ell_j(t)$, the LLM assigned to robot $j$ generates $s_j(t)$. Hereafter, for simplicity of notation, we assume that $\ccalI=\{1,\dots,N\}$ so that $\ccalI(j)=j$. 
We also assume that all agents share the same LLM model; this assumption is made only for ease of notation as it will be discussed later. Next, we first describe how $\ell_j(t)$ is structured and then we discuss how $s_j(t)$ is computed given $\ell_j(t)$. 


\textbf{Prompt Construction:} 
The prompt $\ell_j(t)$ consists of the following parts (see also Fig. \ref{fig:prompt}): (a) \textit{System description} that defines the action space $\ccalA$ that the robot can apply and the objective of the LLM. 
%
%
%
%
(b) \textit{Environment description} that describes the semantic objects that exist in the environment;
(c) \textit{Task description} of the mission $\phi$; 
\textcolor{black}{(d)} \textit{Response structure} describing the desired structure of the LLM output for an example task.
\textcolor{black}{(e)} \textit{History of actions} that includes the sequence of decisions made by all robots up to time $t-1$ and for all robots $1,\dots,j-1$ up to time $t$.
(f) \textit{Current time step and robot index} describing the current time step $t$ and the index $j$ to the robot that is now responsible for picking an action.

\begin{figure}[t] 
\centering
\includegraphics[width=\linewidth]{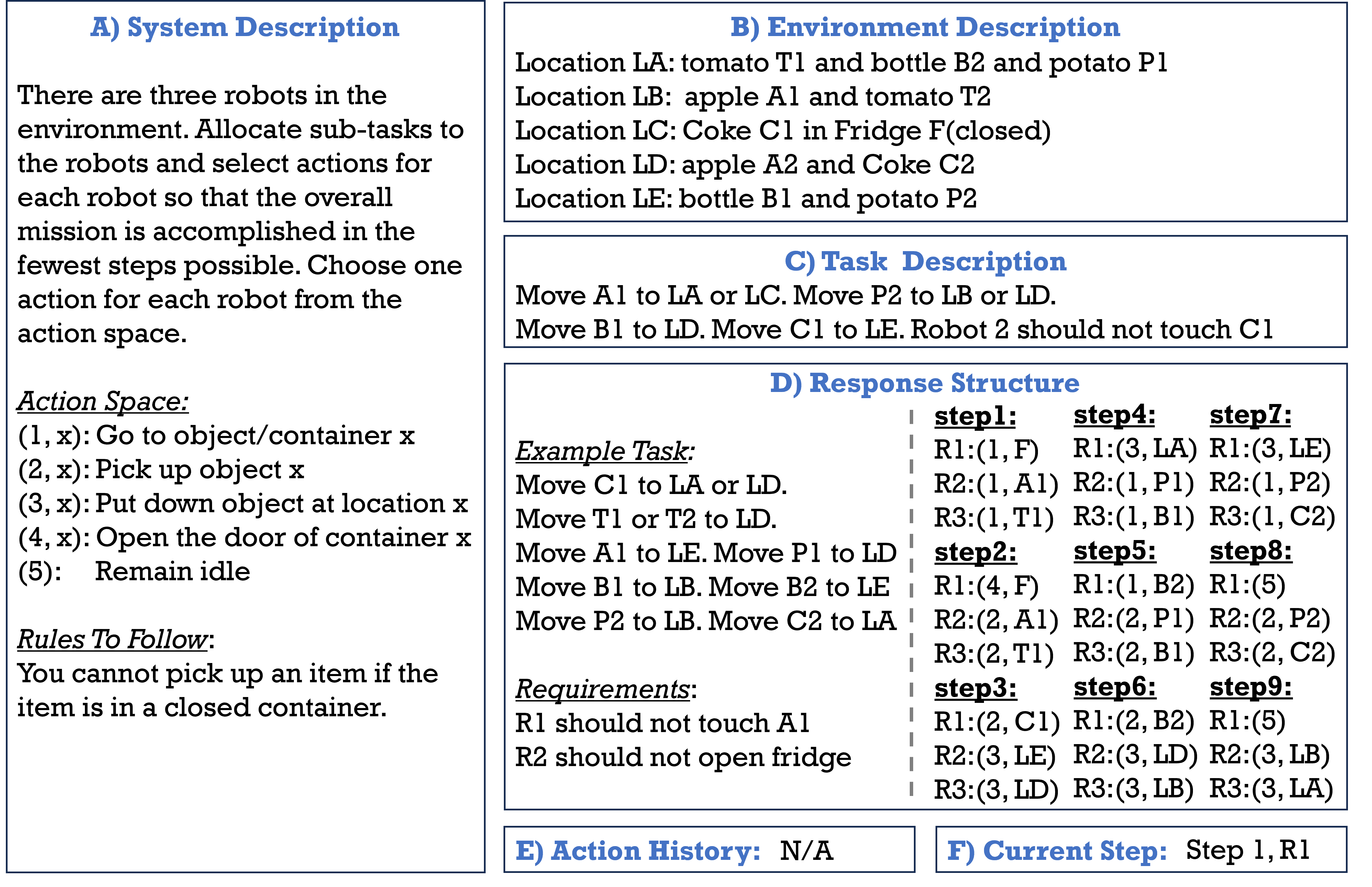} \vspace{-0.5cm}
\caption{\textcolor{black}{Example of the constructed prompt in Section \ref{sec:experiments} \textcolor{black}{for GPT 3.5. } This prompt refers to $t=1$ when the action history is empty.} 
} 
\label{fig:prompt}\vspace{-0.5cm}
\end{figure}

\textbf{Plan Design:} 
Assume that at time $t$, it is robot's $j$ turn to select $s_j(t)$ as per the order captured in $\ccalI$. First, robot $j$ constructs $\ell_j(t)$ using the prompt $\ell_{j-1}(t)$ that the previous robot $j-1$ constructed (or $\ell_{N}(t-1)$, if $j=1$) [line \ref{algo1:nextRbt}-\ref{algo1:constrPrompt}, Alg. \ref{alg:MultiRobotLLM}]. Parts (a)-(c) in $\ell_j(t)$ are the same for all $j$ and $t\geq 1$ (since the environment is static and known); \textcolor{black}{part (d) is the same for all $j$ and $t\geq 1$; part (e) includes the history of actions and can be found in part (e) of $\ell_{j-1}(t)$; and part (f) is constructed as discussed above.}
%
%
The initial prompt $\ell_1(1)$ is manually constructed given a scenario $\xi$ with part (d) being empty [line \ref{algo1:prompt_init}, Alg. \ref{alg:MultiRobotLLM}].

Given $\ell_j(t)$, selection of $s_j(t)$ is represented as an MCQA problem that is solved by the LLM for robot $j$. Specifically, given $\ell_j(t)$ (`question' in MCQA), the LLM $j$ has to pick decision $s_j(t)$ among all available ones included in part (a) (`choices' in MCQA). Given any $s_j(t)\in \mathcal{S}$, LLMs can provide a confidence score $g(s_j(t)|\ell_j(t))$; the higher the score, the more likely the decision $s_j(t)$ is a valid next step to positively progress the task \cite{ahn2022can}.\footnote{If the robots do not share the same LLM model, then $g(s_j(t)|\ell_j(t))$ should be replaced by $g_j(s_j(t)|\ell_j(t))$ throughout the paper.} To get the scores $g(s_j(t)|\ell_j(t))$ for all $s_j(t)\in \mathcal{S}$, we query the model over all potential decisions $s_j$ (i.e., $S=|\ccalS|$ times in total). 
%
Using these scores, a possible approach to select the $s_j(t)$ is by simply choosing the decision with the highest score, i.e., $s_j(t)=\arg\max_{s_j\in\mathcal{S}} g(s_j|\ell_j(t)).$
However, this point-prediction approach is uncertainty-agnostic; note that these scores do not represent calibrated confidence \cite{ren2023robots}. A much preferred approach would be to generate a set of actions (called, hereafter, \textit{prediction set}), denoted by $\ccalC(\ell_j(t))$, that contains the ground truth action with a user-specified high-probability [line \ref{algo1:pred_set}, Alg. \ref{alg:MultiRobotLLM}]. Hereafter, we assume that such prediction sets are provided. We defer their construction to Section \ref{sec:cp} (see \eqref{eq:pred3local}) but we emphasize that these sets are critical to achieving desired $1-\alpha$ mission completion rates.

Given $\ccalC(\ell_j(t))$, we select decisions $s_j(t)$ as follows. If $\ccalC(\ell_j(t))$ is singleton, then we select the action included in $\ccalC(\ell_j(t))$ as it contributes to mission progress with high probability [line \ref{algo1:pickAction}, Alg. \ref{alg:MultiRobotLLM}].\footnote{By construction of the prediction sets, this action coincides with $s_j(t)=\arg\max_{s_j\in\mathcal{S}} g(s_j|\ell_j(t))$; see Section \ref{sec:cp}.} 
Otherwise, robot $j$ seeks assistance from its teammates asking all robots to select new decisions for time $t$ as a per a different order $\ccalI$. If this does not result in singleton prediction sets, robot $j$ asks a human operator to select $s_j(t)$ [lines \ref{algo1:HelpTriggered}-\ref{algo1:HelpHuman2}, Alg. \ref{alg:MultiRobotLLM}]. The process of seeking assistance is discussed in more detail in Section \ref{sec:help}.
%
Once $s_j(t)$ is selected, robot $j$ \textcolor{black}{records $s_j(t)$ into part (e)}.  
With slight abuse of notation, we denote this prompt update by $\ell_{j+1}(t) = \ell_{j}(t) + s_j(t),$ where the summation means concatenation of text [line \ref{algo1:prompt_update}, Alg. \ref{alg:MultiRobotLLM}]. 

Then, robot $j$ sends the resulting prompt  $\ell_j(t)$ to the next robot $j+1$ [line \ref{algo1:sendInfo}, Alg. \ref{alg:MultiRobotLLM}]
We repeat the above procedure sequentially over all robots $j$ as per $\ccalI$. This way, we construct the multi-robot action $\bbs(t)$ [lines \ref{algo1:s}-\ref{algo1:tau}, Alg. \ref{alg:MultiRobotLLM}].
Once $\bbs(t)$ is constructed the current time step is updated to $t+1$. 
The above process is repeated to design $\bbs(t+1)$. 
This way, we generate a plan $\tau=\bbs(1),\dots,\bbs(t),\dots,\bbs(H)$ of horizon $H$. The robots execute the plan $\tau$ once it is fully constructed. We assume that $\tau$ is executed synchronously across the robots; this is a usual assumption in collaborative multi-robot task planning \cite{schillinger2018decomposition,kantaros2020stylus}.

\begin{rem}[Unknown/Dynamic Environments]\label{rem:unknownEnv}
    To relax the assumption of known/static environments, at every time step $t$, the multi-robot action $\bbs(t)$ should be executed as soon as it is computed.
    This execution requires the robots to possess sensor-based controllers $\mathbf{u}_j$ enabling them to implement decisions from $\mathcal{S}$ in unknown and dynamic environments. Once executed, part (b) in the prompt should be updated as per sensor feedback to account e.g., for new or any missing objects. Our future work will focus
on formally relaxing this assumption.
    %
\end{rem}

\vspace{-0.1cm}
\subsection{Seeking Assistance from Teammates and Human Operators}
\label{sec:help}
\vspace{-0.1cm}
As discussed in Section \ref{sec:coorddescent}, \textcolor{black}{if $\mathcal{C}(\ell_j(t))$} is singleton, then this means that the LLM is sufficiently certain that the decision \textcolor{black}{$s_j(t)\in\mathcal{C}(\ell_j(t))$} included in $\mathcal{C}(\ell_j(t)))$ will contribute to making mission progress. Otherwise, the corresponding robot $j$ should seek assistance in the following two ways [lines \ref{algo1:HelpTriggered}-\ref{algo1:HelpHuman2}, Alg. \ref{alg:MultiRobotLLM}]. First, robot $j$ generates a new order $\ccalI$ which is forwarded to all robots along with a message to redesign $\bbs(t)$ from scratch, as per the new ordered set $\ccalI$ [lines \ref{algo1:HelpTriggered}-\ref{algo1:HelpHuman1}, Alg. \ref{alg:MultiRobotLLM}].\footnote{We note that this does not violate the coverage guarantees in Section \ref{sec:cp}, as the distribution is assumed to generate sequences where the order of robots changes over time steps $t$ (within the same sequence); see Section \ref{sec:cp}. 
}
We iterate over various ordered sets $\ccalI$ until singleton prediction sets are generated for all robots. 
%
%
Interestingly, our empirical analysis has shown that $\ccalI$ can potentially affect the size of prediction sets even though this change affects primarily only part (f) in the constructed prompts; see Section \ref{sec:experiments}.
\textcolor{black}{We remark that this step can significantly increase computational complexity of our method, and, therefore, it may be neglected for large robot teams and sets $\ccalS$.} 
The second type of assistance is pursued when a suitable set $\ccalI$ cannot be found with the next $W$ attempts, for some user-specified $W\geq 0$. 
In these cases, the robot $j$ requests help from human operators [lines \ref{algo1:HelpHuman1}-\ref{algo1:HelpHuman2}, Alg. \ref{alg:MultiRobotLLM}]. Particularly, robot $j$ presents to a user the set $\mathcal{C}(\ell_j(t)))$ (along with the prompt $\ell_j(t)$) and asks them to choose one action from it. The user either picks a feasible action from the prediction set for the corresponding robot or decides to halt the operation. \textcolor{black}{In case the prediction set contains multiple feasible solutions, the user selects the one with the highest LLM confidence score; see also Remark \ref{rem:multiFeas2}}. 

\subsection{Constructing Local Prediction Sets using CP}\label{sec:cp}

In this section, we discuss how the prediction sets $\ccalC(\ell_j(t))$, introduced in Section \ref{sec:coorddescent}, are constructed, given a required task success rate $1-\alpha$ (see Problem \ref{pr:pr1}). 
%
The MCQA setup allows us to apply conformal prediction (CP) to construct these sets. 
To illustrate the challenges of their construction, we consider the following cases: (i) single-robot \& single-step plans 
and (ii) multi-robot \& multi-step plans.

\textbf{Single-robot \& Single-step Plans:} Initially, we focus only scenarios with $N=1$ and $H=1$. 
%
First, we sample $M$ independent scenarios from $\ccalD$. We refer to these scenarios as calibration scenarios. For each calibration scenario $i\in\{1,\dots,M\}$, we construct its equivalent prompt $\ell_{j,\text{calib}}^i$ associated with (the single) robot $j$. For each prompt, we (manually) compute the ground truth plan $\tau_{\text{calib}}^{i}=s_{j,\text{calib}}^{i}(1)$ accomplishing this task.
\textcolor{black}{For simplicity, we assume that there exists a unique correct decision $s_{j,\text{calib}}^{i}$ for each $\ell_{j,\text{calib}}^i$; this assumption is relaxed in Remark \ref{rem:multipleFeas1}.}
%
Hereafter, we drop the dependence of the robot decisions and prompts on the time step, since we consider single-step plans. This way we construct a calibration dataset $\ccalM=\{(\ell_{j,\text{calib}}^i,\tau_{\text{calib}}^{i})\}_{i=1}^M$. 

Consider a new scenario drawn from $\ccalD$, called validation/test scenario. We convert this scenario into its equivalent prompt $\ell_{j,\text{test}}$. 
Since the calibration and the validation scenario are i.i.d., 
%
CP can generate a prediction set $\ccalC(\ell_{j,\text{test}})$ of decisions $s_j$ containing the correct one, denoted by $s_{j,\text{test}}$, with probability greater than $1-\alpha$, i.e.,  
\begin{equation}\label{eq:CP1}
P(s_{j,\text{test}}\in \mathcal{C}(\ell_{j,\text{test}})) \geq 1-\alpha,
\end{equation}
where $\textcolor{black}{1-\alpha\in(0,1)}$ is user-specified.
To generate $\mathcal{C}(\ell_{j,\text{test}})$, CP first uses the LLM’s confidence $g$ to compute the set of
non-conformity scores (NCS) $\{r^i=1-g(s_{j,\text{calib}}^i~|~\ell_{j,\text{calib}}^i)\}_{i=1}^M$ over the calibration set.
The higher the score is, the less each calibration point conforms to the data used for training the LLM. Then we perform calibration by computing the $\frac{(M+1)(1-\alpha)}{M}$ empirical quantile of $r_1,\dots,r_M$ denoted by $q$. Using $q$, CP constructs the prediction set 
\begin{equation}\label{eq:pred1}
\mathcal{C}(\ell_{j,\text{test}})=\{s_j\in \ccalS~|~g(s_j|\ell_{j,\text{test}})>1-q\},
\end{equation}
that includes all decisions that the model is at least $1-q$ confident in. This prediction set is used to select decisions as per Section \ref{sec:coorddescent} for single-robot single-step tasks. 
%
By construction of the prediction sets the decision $s_j=\arg\max_{s_j\in\mathcal{S}} g(s_j|\ell_{j,\text{test}})$ belongs to $\mathcal{C}(\ell_{j,\text{test}})$. 

\textbf{Multi-robot \& Multi-step Plans:} Next, we generalize the above result to the case where $N\geq1$ and $H\geq 1$.
%
Here we cannot apply directly \eqref{eq:CP1}-\eqref{eq:pred1} to compute individual sets $\ccalC(\ell_{j,\text{test}}(t))$ for each robot, as this violates the i.i.d. assumption required to apply CP. 
Particularly, the distribution of prompts $\ell_{j,\text{test}}(t)$ depends on the previous prompts i.e., the multi-robot decisions at $t'<t$ as well as the decisions of robots $1,\dots,j-1$ at time $t$. 
Inspired by \cite{ren2023robots}, we address this challenge by (i) lifting the data to sequences, and (ii) performing calibration at the sequence level using a carefully designed NCS. 
%
%

We sample $M\geq1$ independent calibration scenarios from $\ccalD$. Each scenario corresponds to various tasks $\phi_i$, numbers of robots $N_i$, and mission horizons $H_i$. Each task $\phi_i$ is broken into a sequence of $T_i=H_i\cdot N_i\geq 1$ prompts defined as in Section \ref{sec:coorddescent}. 
This sequence of $T_i$ prompts is denoted by
\begin{align}\label{eq:seqProm}
&\bar{\ell}_{\text{calib}}^i=\underbrace{\ell_{1,\text{calib}}^i(1),\dots,\ell_{j,\text{calib}}^i(1),\dots,\ell_{N_i,\text{calib}}^i(1)}_{t=1},\dots,\nonumber\\&
\underbrace{\ell_{1,\text{calib}}^i(t'),\dots,\ell_{j,\text{calib}}^i(t'),\dots,\ell_{N_i,\text{calib}}^i(t')}_{t=t'},\dots,\nonumber\\&
\underbrace{\ell_{1,\text{calib}}^i(H_i),\dots,\ell_{j,\text{calib}}^i(H_i),\dots,\ell_{N_i,\text{calib}}^i(H_i)}_{t=H_i}.
\end{align}
where, by construction, each prompt $\bar{\ell}_{\text{calib}}^i$ contains a history of ground truth decisions made so far. We define the corresponding sequence of ground truth decisions as:\footnote{
%
%
The distribution $\ccalD$ induces a distribution over data sequences \eqref{eq:seqProm} \cite{ren2023robots}. These data sequences are equivalent representations of the sampled scenarios augmented with the correct decisions.
%
Observe that in these sequences the order of robots is determined by an ordered set $\ccalI_i$ (generated by the induced distribution). 
For ease of notation, in this section, we assume that the ordered set is defined as $\ccalI_i=\{1,\dots,N_i\}$ and $\ccalI_{\text{test}}=\{1,\dots,N_\text{test}\}$ for all calibration and validation sequences, respectively, and for all $t$. \textcolor{black}{However, in general, these ordered sets do not need to be the same across the calibration sequences and/or across the time steps $t$ within the $i$-th calibration sequence. 
For instance, the induced distribution may randomly pick a set $\ccalI$ from a finite set of possible sets $\ccalI$ for each time $t$. Then, in Section \ref{sec:help}, when `help from robots' is needed, a new set $\ccalI$ from this finite set can be randomly selected.
}}
\begin{align}\label{eq:Plans3}
&\tau_{\text{calib}}^i=\underbrace{s_{1,\text{calib}}(1),\dots,s_{j,\text{calib}}(1),\dots,s_{N_i,\text{calib}}(1)}_{=\bbs_{\text{calib}}(1)},\dots,\nonumber\\&
\underbrace{s_{1,\text{calib}}(t'),\dots,s_{j,\text{calib}}(t'),\dots,s_{N_i,\text{calib}}(t')}_{=\bbs_{\text{calib}}(t')},\dots,\nonumber\\&
\underbrace{s_{1,\text{calib}}(H_i),\dots,s_{j,\text{calib}}(H_i),\dots,s_{N_i,\text{calib}}(H_i)}_{=\bbs_{\text{calib}}(H_i)}.
\end{align}
\noindent
This gives rise to a calibration set  $\ccalM=\{(\bar{\ell}_{\text{calib}}^i,\tau_{\text{calib}}^i)\}_{i=1}^M$. 
%
%
%
\textcolor{black}{As before, for simplicity, we assume that each context $\bar{\ell}_{\text{calib}}^i$ has a unique correct plan $\tau_{\text{calib}}^i$; this assumption is relaxed in Remark \ref{rem:multipleFeas1}.}
%
%
We denote by $\bar{\ell}_{\text{calib}}^i(k)$ and $\tau_{\text{calib}}^i(k)$ the $k$-th entry in $\bar{\ell}_{\text{calib}}^i$ and $\tau_{\text{calib}}^i$, respectively, where $k\in\{1,\dots,T_i\}$. Note that the iterations $k$ are different from the time steps $t\in\{1,\dots,H_i\}$. 
\textcolor{black}{Each iteration $k$ refers to the prompt/decision corresponding to a specific robot index $j$ at time step $t$.} 

Next, for each calibration sequence, we define the NCS similarly to the single-robot-single-step plan case. 
Specifically, the NCS of the $i$-th calibration sequence, denoted by $\bar{r}_i$, is computed based on the lowest NCS over $k\in\{1,\dots,T_i\}$, i.e., 
\begin{equation}
\bar{r}_i=1-\bar{g}(\tau_{\text{calib}}^i|\bar{\ell}_\text{calib}^i),
\end{equation}
where
\begin{equation}\label{eq:min_seqpred}
\bar{g}(\tau_{\text{calib}}^i~|~\bar{\ell}_\text{calib}^i)=\min_{k\in\{1,\dots, T\}}g(s_\text{calib}^i(k)~|~\bar{\ell}_\text{calib}^i(k)).
\end{equation}
%

Consider a new validation scenario drawn from $\ccalD$ associated with a task $\phi_{\text{test}}$ that is defined over $N_{\text{test}}$ robots and horizon $H_{\text{test}}$ corresponding to a sequence of prompts $$\bar{\ell}_{\text{test}}=\bar{\ell}_{\text{test}}(1),\dots,\bar{\ell}_{\text{test}}(k),\dots,\bar{\ell}_{\text{test}}(T_{\text{test}}),$$
where $T_{\text{test}}=H_\text{test}\cdot N_{\text{test}}$. 
%
Using the set $\ccalM$, CP can generate a prediction set $\bar{\ccalC}(\bar{\ell}_{\text{test}})$ of plans $\tau$, containing the correct one, denoted by $\tau_{\text{test}}$,  with probability greater than $1-\alpha$, i.e., 
\begin{equation}\label{eq:CP3}
P(\tau_{\text{test}}\in \bar{\mathcal{C}}(\bar{\ell}_{\text{test}})) \geq 1-\alpha.
\end{equation}
The prediction set $\bar{\mathcal{C}}(\bar{\ell}_{\text{test}})$ is defined as: 
\begin{equation}\label{eq:pred3}
\bar{\mathcal{C}}(\bar{\ell}_{\text{test}})=\{\tau~|~\bar{g}(\tau|\bar{\ell}_{\text{test}})>1-\bar{q}\}, 
\end{equation}
where $\bar{q}$ is the $\frac{(M+1)(1-\alpha)}{M}$ 
empirical quantile of $\bar{r}_1,\dots,\bar{r}_M$.
%
The size of the prediction sets can be used to evaluate uncertainty of the LLM. Specifically, if $\bar{\mathcal{C}}(\bar{\ell}_{\text{test}})$ is singleton then this means that the LLM is certain with probability at least $1-\alpha$ that the designed multi-robot plan accomplishes the assigned task.

\underline{On-the-fly and Local Construction:}
Notice that $\bar{\mathcal{C}}(\bar{\ell}_{\text{test}})$ in \eqref{eq:pred3} is constructed after the entire sequence $\bar{\ell}_{\text{test}}=\bar\ell_{\text{test}}(1),\dots,\bar\ell_{\text{test}}(T_{\text{test}})$ is obtained. 
%
%
%
However, at test time, we do not see the entire sequence of prompts all at once; instead, the contexts $\bar{\ell}_{\text{test}}(k)$ are revealed sequentially over iterations $k$, as the robots pick their actions. 
Thus, next we construct the prediction set $\bar{\mathcal{C}}(\bar{\ell}_{\text{test}})$ on-the-fly and incrementally, using only the current and past information, as the robots take turns in querying their LLMs for action selection.
At every iteration $k\in\{1,\dots,T\}$, we construct the local prediction set associated with a robot $j$ and a time step $t$,  
\begin{equation}\label{eq:pred3local}
    \textcolor{black}{\mathcal{C}(\bar{\ell}_{\text{test}}(k))=\{\tau_{\text{test}}(k)~|~g(\tau_{\text{test}}(k)|\bar{\ell}_{\text{test}}(k)))>1-\bar{q}\},}
\end{equation}
where $g$ refers to the (uncalibrated) LLM score used in Section \ref{sec:coorddescent}. The prediction set $ \mathcal{C}(\bar{\ell}_{\text{test}}(k)))$ in \eqref{eq:pred3local}  should be used by the respective robot $j$ to select a decision $s_j(t)$. Essentially, $\mathcal{C}(\bar{\ell}_{\text{test}}(k)))$ corresponds to the set $\mathcal{C}(\ell_{j}(t)))$ used in Section \ref{sec:coorddescent} and in [line \ref{algo1:pred_set}, Alg. \ref{alg:MultiRobotLLM}].\footnote{For instance, the set $\ccalC(\bar{\ell}_{\text{test}}(k))$ with $k=N_{\text{test}}+1$ refers to the prediction set of robot $\ccalI(1)$ at time $t=2$.}
%
%
%
%
%
%
%
As it will be shown in Section \ref{sec:theory}, it holds that $\bar{\mathcal{C}}(\bar{\ell}_{\text{test}})=\hat{\ccalC}(\bar{\ell}_{\text{test}})$,
where
\begin{equation}\label{eq:causalPred3}
    \hat{\ccalC}(\bar{\ell}_{\text{test}})=\mathcal{C}(\bar{\ell}_{\text{test}}(1))\times\dots\times\mathcal{C}(\bar{\ell}_{\text{test}}(k))\dots\times\mathcal{C}(\bar{\ell}_{\text{test}}(T_{\text{test}})).
\end{equation}

%

\begin{rem}[Prediction sets]\label{rem:datasetInd}
%
\textcolor{black}{Construction of prediction sets requires generating a new calibration set for every new test scenario to ensure the desired coverage level. A dataset-conditional guarantee which holds for a fixed calibration set can also be applied \cite{vovk2012conditional}. Also, obtaining a
non-empty prediction set for a fixed $\alpha$, requires the size of the calibration dataset to satisfy: $M \geq \ceil[\big]{(M + 1)(1 - \alpha)}$.} 
\end{rem}

\begin{rem}[Efficiency Benefits]\label{rem:API}
   \textcolor{black}{Given a quantile $\bar{q}$ and the scores for each decision, the complexity of constructing the prediction sets $\ccalC(\ell_j(t))$ for all robots $j$ at time $t$ is $O(N \cdot |\ccalS|)$. This is notably more efficient than the centralized approach of Section \ref{sec:centralized}, where the corresponding complexity of constructing multi-robot/global prediction sets is $O(|\ccalS|^N)$.}
\end{rem}


\begin{rem}[Multiple Feasible Solutions]\label{rem:multipleFeas1}
\textcolor{black}{The above CP analysis assumes that $\ccalD$ generates scenarios with a unique solution. To relax this assumption, the key modification lies in the construction of the calibration dataset. Specifically, for each calibration mission scenario, among all feasible plans, we select the one constructed by picking the decision with the highest LLM confidence score at each planning step. CP can then be applied as usual, yielding prediction sets that contain, with user-specified probability, the plan $\tau_{\text{test}}$ consisting of the decisions with the highest LLM confidence scores among all feasible decisions. This is formally shown in Appendix \ref{app:proof} by following the same steps as in \cite{ren2023robots}. 
}
\end{rem}
\begin{rem}[i.i.d. Assumption]\label{rem:robustCP}
\textcolor{black}{The guarantees in \eqref{eq:CP3} hold only if
calibration and test scenarios are i.i.d. This can be relaxed by employing robust CP to obtain valid prediction sets for all distributions $\ccalD'$ that are `close' to $D$ (based on the $f-$divergence) \cite{cauchois2024robust}.} 
\end{rem}


\begin{rem} [Dataset-Conditional Guarantee]\label{rem:datasetInd}
 The probabilistic guarantee in \eqref{eq:CP3} is \textit{marginal} in the sense that the probability is over both the sampling of the calibration set $\ccalM$ and the validation point $\bar{\ell}_{\text{test}}$. Thus, a new calibration set will be needed for every single test data point $\bar{\ell}_{\text{test}}$ to ensure the desired coverage level. A dataset-conditional guarantee which holds for a fixed calibration set can also be applied \cite{vovk2012conditional}.
\end{rem}

\section{Theoretical Mission Success Rate Guarantees}\label{sec:theory}


In this section, we show that the proposed language-based multi-robot planner presented in Algorithm \ref{alg:MultiRobotLLM} achieves user-specified $1-\alpha$ mission success rates. 
To show this, we need first to state the following result; the proofs of the following results are adapted from \cite{ren2023robots}. 
\begin{prop}\label{prop:eq_prop}
The prediction set $\bar{\ccalC}(\bar{\ell}_{\text{test}})$ defined in \eqref{eq:pred3} is the same as the on-the-fly constructed prediction set $\hat{\ccalC}(\bar{\ell}_{\text{test}})$ defined in \eqref{eq:causalPred3}, i.e.,  $\bar{\ccalC}(\bar{\ell}_{\text{test}})=\hat{\ccalC}(\bar{\ell}_{\text{test}})$.
\end{prop}


\begin{proof}
It suffices to show that if a multi-robot plan $\tau$ belongs to $\bar{\ccalC}(\bar{\ell}_{\text{test}})$ then it also belongs to $\hat{\ccalC}(\bar{\ell}_{\text{test}})$ and vice-versa. 
First, we show that if $\tau\in\bar{\ccalC}(\bar{\ell}_{\text{test}})$ then $\tau\in\hat{\ccalC}(\bar{\ell}_{\text{test}})$. Since $\tau\in \bar{\mathcal{C}}(\bar{\ell}_{\text{test}})$, then we have that $\min_{k\in\{1,\dots,T_{\text{test}}\}}g(\tau(k) | \bar{\ell}_{\text{test}}(k) )>1-\bar{q}$ due to \eqref{eq:min_seqpred}. This means that $g(\tau(k) | \bar{\ell}_{\text{test}}(k) )>1-\bar{q}$, for all  $k\in \{1,\dots,T_{\text{test}}\}$. Thus, $\tau(k)\in\ccalC(\bar{\ell}_{\text{test}}(k))$, for all  $k\in \{1,\dots,T_{\text{test}}\}$. By definition of $\hat{\ccalC}(\bar{\ell}_{\text{test}})$ in \eqref{eq:causalPred3}, this implies that $\tau\in\hat{\ccalC}(\bar{\ell}_{\text{test}})$. These steps hold in the other direction too showing that if $\tau\in\hat{\ccalC}(\bar{\ell}_{\text{test}})$ then $\tau\in\bar{\ccalC}(\bar{\ell}_{\text{test}})$. 
\end{proof}


\begin{theorem}[Mission Success Rate]\label{thm1}
Assume that prediction sets are constructed on-the-fly/causally with coverage level $1-\alpha$ and that the robots seek help from a user whenever the local prediction set $\ccalC(\bar{\ell}_{\text{test}}(k))$ - defined in \eqref{eq:pred3local} - is not singleton after $W$ attempts; see Section \ref{sec:help}. (a) \textcolor{black}{Assuming error-free execution of the designed plans}, the task completion rate over new test scenarios (and the randomness of the calibration sets) drawn from $\ccalD$ is at least $1-\alpha$. (b) If $\bar{g}(\tau|\bar{\ell}_{\text{test}})$, used in \eqref{eq:pred3local}, models true conditional probabilities, \textcolor{black}{then the average amount of user help}, modeled by the average size of the prediction sets  $\bar{\ccalC}(\bar{\ell}_{\text{test}})$, is minimized among possible prediction schemes that achieve $1-\alpha$ mission success rates.
\end{theorem}

\begin{proof}
(a) To show this result, we consider the following cases.
Case I: We have that $|\mathcal{C}(\bar{\ell}_{\text{test}}(k))|=1$, $\forall k\in\{1,\dots,T_{\text{test}}\}$ (with or without the robots `helping' each other) and $\tau_{\text{test}}\in\hat{C}(\bar{\ell}_{\text{test}})$ where $\tau_{\text{test}}$ is the ground truth plan and 
$\hat{C}(\bar{\ell}_{\text{test}})$ is defined as in \eqref{eq:causalPred3}. In this case, the robots will select the correct plan.
%
Case II: We have that $\tau_{\text{test}}\in\hat{C}(\bar{\ell}_{\text{test}})$, regardless of the size of the local prediction sets $\mathcal{C}(\bar{\ell}_{\text{test}}(k))$, $\forall k\in\{1,\dots,T_{\text{test}}\}$.
Here, the robots will select the correct plan assuming users who faithfully provide help; \textcolor{black}{otherwise a distribution shift may occur as calibration sequences are constructed using correct decisions.}
Case III: We have that $\tau_{\text{test}}\notin\hat{\ccalC}(\bar{\ell}_{\text{test}})$. The latter means that there exists at least one iteration $k$ such that $\tau_{\text{test}}(k)\notin{\ccalC}(\bar{\ell}_{\text{test}}(k))$. In this case, the robots will compute an incorrect plan.
Observe that the probability that either Case I or II will occur is equivalent to the probability $P(\tau_{\text{test}}\in \hat{{\mathcal{C}}}(\bar{\ell}_{\text{test}}))$. Due to Prop. \ref{prop:eq_prop} and \eqref{eq:CP3}, we have that $P(\tau_{\text{test}}\in \hat{\mathcal{C}}(\bar{\ell}_{\text{test}})) \geq 1-\alpha$. Thus, either of Case I and II will occur with probability that is at least equal to $1-\alpha$. Since Cases I-III are mutually and collectively exhaustive, we conclude that the probability that Case III will occur is less than $\alpha$. This implies that the mission success rate is at least $1-\alpha$.
%
(b) This result holds directly due to Th. 1 in \cite{sadinle2019least}.
\end{proof}

\begin{rem}[Multiple Feasible Solutions]\label{rem:multiFeas2}
    \textcolor{black}{The above results hold in case of multiple feasible solutions too. The proof of Th. \ref{thm1} follows the same steps, with the only difference being that (i) the correct plan $\tau_{\text{test}}$ refers to the feasible plan comprising the decisions with the highest LLM confidence scores at each step, among all feasible plans, and (ii) when user help is required, the user selects the feasible action with the highest LLM confidence score from the prediction set. 
    Thus, the robot must also return to the user the scores for each action in the set.}
\end{rem}

\section{Experiments}\label{sec:experiments}

\begin{figure*}[t] 
\centering
\includegraphics[width=\linewidth]{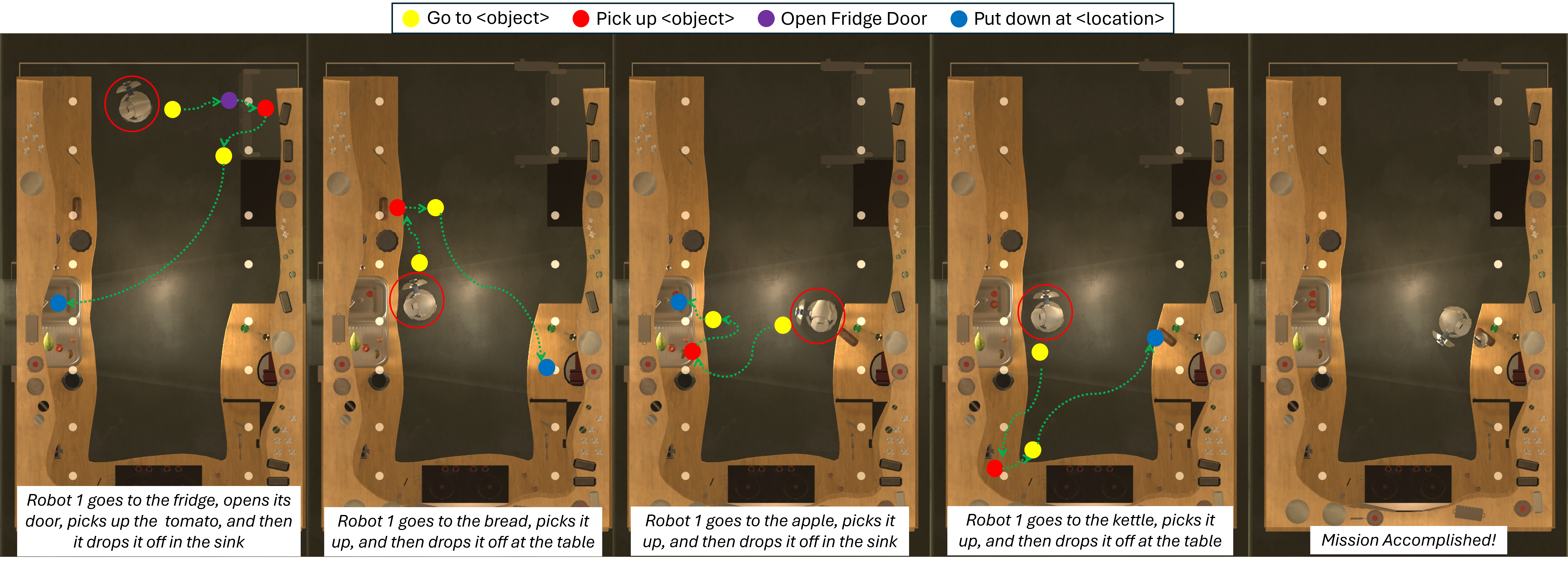}\vspace{-0.5cm}  
\caption{Execution of a single-robot plan, generated by S-ATLAS, for a mission: `Deliver a tomato and an apple to the sink. Also, deliver the kettle and the bread to the table'. 
} 
\label{fig:snapshot_single}\vspace{-0.3cm}
\end{figure*}

\begin{figure*}[t] 
\centering
\includegraphics[width=\linewidth]{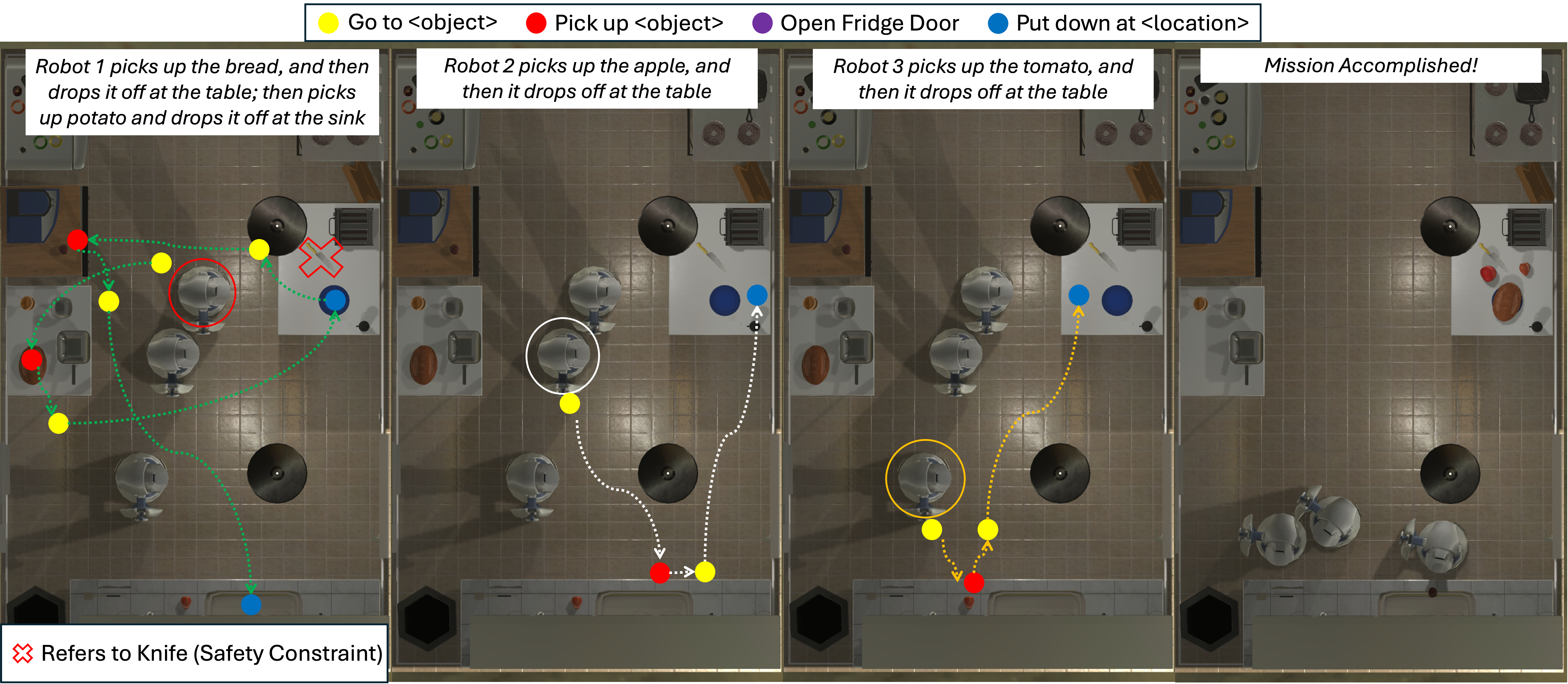}\vspace{-0.5cm} 
\caption{Execution of a three-robot plan, generated by S-ATLAS, for the mission: `Deliver the bread, apple and tomato to the table; Deliver the potato to the sink. Also, Robot 1 should never pick up a knife'. The first three snapshots left illustrate the plans, generated by S-ATLAS, followed by each robot to achieve certain sub-tasks of the mission. The fourth snapshot shows the environment when all the sub-tasks are completed. } 
\label{fig:snapshot_multi}\vspace{-0.3cm}
\end{figure*}

To illustrate the performance of the proposed planner, we conduct extensive comparative experiments considering home service robot tasks on the AI2THOR simulator \cite{kolve2017ai2}. In Section 
\ref{sec:cpresult}, we empirically validate the theoretical mission completion rates of our planner stated in Theorem \ref{thm1}. \textcolor{black}{In Section \ref{sec:knownoComp}, we compare our method against the centralized baselines from Section \ref{sec:centralized}.} Our comparisons demonstrate that our planner is more computationally efficient and achieves significantly lower help rates. Additional comparisons against LLM-based multi-robot planners that do not utilize CP and do not allow robots to ask for help can be found in Appendix \ref{app:Comp}. 
%
%
Finally, in Section \ref{sec:effectOrder}, we demonstrate how assistance from robot teammates can potentially help robots make more certain decisions. 
In all case studies, we pick GPT-3.5,  \textcolor{black}{Llama-2-7b}, and Llama-3-8b as the LLM for all robots.

\subsection{Empirical Validation of Mission Success Rates}\label{sec:cpresult}

%
%

%
We consider home service tasks defined over the following $12$ semantic objects with labels `Apple', `Kettle', `Tomato', `Bread', `Potato', and `Knife'. The set $\ccalA$ includes the actions \text{`go to, grab object, put object down, open door, remain idle'}.  The action `remain idle' is useful when  all sub-tasks in $\phi$ have been assigned to other robots or when the task has been accomplished before the given horizon $H$. 
%
The number of decisions that the LLM can pick from is $|\mathcal{S}|=28$. 

\textcolor{black}{We also generate $110$ (validation) scenarios $\xi$ with $N\in\{1, 3, 10, 15\}$ in various semantic environments. Each mission consists of $K$ sub-tasks and at most one  safety requirement. Each sub-task requires moving a specific semantic object, possibly located inside closed containers (e.g., drawers and fridge), to either a desired location or one of several possible destinations. 
The safety constraint requires certain robots to avoid approaching/grabbing specific objects. We  select $K\leq 4$ for $N=1$, $4\leq K\leq 8$ for $N\in\{3,10\}$, and $K=10$ when $N=15$. Observe that $N\leq K$ implies that a robot will have to accomplish at least one sub-task while $N\leq K$ means that at least one robot will stay idle throughout the mission.} \textcolor{black}{Notice that the considered tasks may have multiple feasible plans and, therefore, CP is applied as in Remarks \ref{rem:multipleFeas1} and \ref{rem:multiFeas2}.} Examples of a task for $N=1$ and $N=3$, along with the plans designed by S-ATLAS, coupled with GPT 3.5, are shown in Figures \ref{fig:snapshot_single} and \ref{fig:snapshot_multi}, respectively; see also the demonstrations in \cite{DemoSafeMultiLLM}. 
%
%

In what follows, we empirically validate the mission success rates discussed in Section \ref{sec:theory} on our validation scenarios. Specifically, we investigate if we can achieve desired mission success rate $1-\alpha$. For each validation scenario, we generate $30$ calibration sequences to construct prediction sets. Then, for each scenario, we compute the plan using Alg. \ref{alg:MultiRobotLLM} with $W=0$ (to minimize API calls) and (manually) check if the robot plan is the ground truth plan. We compute the ratio of how many of the corresponding $110$ generated plans are the ground truth ones.  
We repeat this $50$ times. The average ratio across all experiments is the (empirical) mission success rate. When $1-\alpha=0.90$ and $1-\alpha=0.95$, the mission success rate was $94.59\%$ and $96.57\%$ respectively, using GPT 3.5, validating Theorem \ref{thm1}. 
\textcolor{black}{The average runtime to design a plan  was $0.4$ minutes (by querying GPT 3.5 $S\cdot N$ times, in parallel, to design $\bbs(t)$).}
When $1-\alpha=0.90$ and $1-\alpha=0.95$, the average percentage of local prediction sets $\ccalC(\ell_{\text{test}}(k))$ that were singletons was $95.56\%$ and $91.90\%$, respectively. 
This equivalently means that when $1-\alpha=0.90$ and $1-\alpha=0.95$, help was requested in making $4.44\%$ and $8.1\%$ of all decisions $s_j(t)$, respectively. As expected, the frequency of help requests increases as the desired mission success rate increases. Help rates also depend on the choice of the LLM. \textcolor{black}{For instance, when our planner is paired with Llama 3-8b, which is significantly smaller and, therefore, less effective than GPT 3.5, the help rate increases to $8.9\%$ and $15.6\%$ for $1-\alpha=0.90$ and $1-\alpha=0.95$, respectively. When Llama 2-7b, which preceded Llama 3-8b, is considered, the help rates increase to 25\% and 40\% for $1-\alpha=0.90$ and $1-\alpha=0.95$, respectively.}

\begin{table}[t]
\centering
\begin{tabular}{|c|cc|cc|}
\hline
\multirow{2}{*}{}         & \multicolumn{2}{c|}{Llama2-7b}       & \multicolumn{2}{c|}{Llama3-8b}          \\ \cline{2-5} 
                          & \multicolumn{1}{c|}{Ours}   & \textcolor{black}{B} & \multicolumn{1}{c|}{Ours}    & \textcolor{black}{B}   \\ \hline
Runtime \textcolor{black}{for plan design} (mins)         & \multicolumn{1}{c|}{3}      & 23     & \multicolumn{1}{c|}{4}       & 24       \\ \hline
Help Rate($1-\alpha=80\%$) & \multicolumn{1}{c|}{$16\%$} & $62\%$ & \multicolumn{1}{c|}{$1.4\%$} & $13.4\%$ \\ \hline
Help Rate($1-\alpha=90\%$) & \multicolumn{1}{c|}{$29\%$} & $88\%$ & \multicolumn{1}{c|}{$8\%$}   & $30\%$   \\ \hline
\end{tabular}
\caption{\textcolor{black}{Comparison of S-ATLAS \textcolor{black}{against the baseline (B)} for $N=2$.}}\vspace{-1.1cm}
\label{table:knowno_comp}
\end{table}


\subsection{Comparisons against Conformal Centralized Baselines}
\label{sec:knownoComp}

\textcolor{black}{First, we compare our planner against the centralized baseline discussed in Section \ref{sec:centralized}, in terms of help rates and computational efficiency. Evaluating the baseline in multi-robot scenarios is computationally challenging and impractical due to the large number $|\ccalS|^N$ of multi-robot decisions that need to be considered; e.g., for $N=15$, we have $|\ccalS|^N=5.0977\cdot 10^{21}$. Specifically, the baseline designs $\bbs(t)$ by solving a single MCQA problem with $|\ccalS|^N$ choices. In contrast, our planner solves sequentially $N$ MCQA problems with $|\ccalS|$ choices each. Also, recall that application of CP requires obtaining the LLM confidence scores for each option in the MCQA problem. To obtain them, our implementation requires one LLM query per option. Thus, the baseline requires $|\ccalS|^N$ queries to apply CP at time $t$ (while ours require $N\cdot|\ccalS|$). 
This results in prohibitively high API costs (using GPT 3.5) and further compromises the computational efficiency of the baseline (using either model). 
%
To enable comparisons, we consider small teams with $N=2$ robots and $|\ccalS|=28$ using  Llama 2-7b \textcolor{black}{and Llama 3-8b; see Table \ref{table:knowno_comp}}. We generate $20$ test scenarios and, for each scenario, we collect $20$ calibration sequences. We compute plans using S-ATLAS (with $W=0$) and the baseline. We repeat this $50$ times. \textcolor{black}{Both methods were exposed to the same validation and calibration data.}  \textcolor{black}{Observe in Table \ref{table:knowno_comp} that} our method (i) requires significantly lower help rates from users than the baseline,\footnote{\textcolor{black}{We observed that larger sets of options tend to increase help rates.}} and (ii) computes plans faster.\footnote{\textcolor{black}{Note that these runtimes are implementation-specific. For example, our method (and, similarly, the baseline) can be accelerated by parallelizing the extraction of $|\ccalS|$ LLM confidence scores during the construction of $\ccalC(\ell_j(t))$.
}} \textcolor{black}{We attribute these to the distributed nature of our planner enabling it to solve `smaller' MCQA scenarios, compared to the baseline, to design $\bbs(t)$.}}\footnote{\textcolor{black}{Our implementation of S-ATLAS, paired with Llama 3-8b, in an unseen test scenario with 
$N=2$ robots, can be found on this \href{https://drive.google.com/file/d/1Ds5wulsNRdXPb1KE6gJyMU9LZlUkcs1N/view?usp=sharing}{link}.}} 

\textcolor{black}{Second, we evaluate the baseline on the previously considered case studies, using the more efficient setup discussed in Rem. \ref{rem:baseline2}.  We query GPT-3.5 to generate a set of $X=4$ multi-robot decisions as in \cite{ren2023robots}. Our empirical results show that the LLM's ability to generate options containing a valid choice decreases as $N$ increases. For $N=1, 2, 3$ and $10$, valid options were included in the LLM-generated set in only $95\%$, $85\%$, $48\%$,  and $20\%$ of the planning steps, leading to incomplete plans.} \textcolor{black}{Note that increasing $X$ does not necessarily yield better results; e.g., for $N=3$, the corresponding percentages for $X=6, 8,$ and $10$  are $52\%$, $48\%$, and $48\%$.} 

\subsection{Asking for Help from Robot Teammates}\label{sec:effectOrder}

In this section, we select five scenarios from Section \ref{sec:cpresult}, with $N=3$, where non-singleton prediction sets were constructed using GPT 3.5. We show that re-computing $\bbs(t)$ using a different set $\ccalI$ can potentially result in smaller prediction sets. In each of these scenarios, we select $W=1$, and we observe that two scenarios switched to singleton prediction sets. 
In the first scenario, the task $\phi$ requires the robots to move a tomato and a water bottle to the sink, put the kettle to stove burner, and the bread at the table. We initialize the ordered set of robot indices as $\ccalI = \{1, 2, 3 \}$. In this scenario, the bottle of water is inside the closed fridge and, therefore, grabbing the bottle requires first opening the fridge door. The prediction set $\ccalC(\ell_1(1))$ for the robot $\ccalI(1)=1$ at $t=1$ is non-singleton and contains the following actions: {(go to location of water bottle), (go to the location of the fridge)}. In this case, robot $1$ (randomly) generates a new ordered set: $\ccalI' = \{2, 1, 3\}$ and asks all robots to select their decisions as per $\ccalI'$. Thus,  the robot with index $\ccalI'(1)=2$ will select first an action. In this case, the corresponding prediction set is singleton defined as {(go to the location of the fridge)}. Recall from Section \ref{sec:help} that if a singleton prediction set cannot be constructed after trying $W$ different sets $\ccalI$, then help from a user will be requested. Similar observations were made in the second scenario requiring the robots to move multiple objects at desired locations. We note again here that changing the set $\ccalI$ may not always result in smaller prediction sets, as this minimally affects the prompts. 

\section{Conclusions, Limitations, and Future Work}\label{sec:concl}

This paper proposed S-ATLAS, a new decentralized planner for language-instructed robot teams. We showed both theoretically and empirically that the proposed planner can achieve desired mission success rates while asking for help from users only when necessary. 
In our future work we plan to relax the assumption of perfect robot skills \textcolor{black}{by leveraging CP to reason simultaneously about uncertainties arising from LLMs and robot skill execution.}
\bibliographystyle{IEEEtran}
\bibliography{YK_bib.bib}


\begin{table*}[t]
\centering
\begin{tabular}{|c|ccc|ccc|ccc|}
\hline
\multirow{2}{*}{N} & \multicolumn{3}{c|}{GPT 3.5}                                             & \multicolumn{3}{c|}{Llama 3-8b}                                          & \multicolumn{3}{c|}{Llama 2-7b}                                    \\ \cline{2-10} 
                   & \multicolumn{1}{c|}{Ours}     & \multicolumn{1}{c|}{CMAS}     & DMAS     & \multicolumn{1}{c|}{Ours}     & \multicolumn{1}{c|}{CMAS}     & DMAS     & \multicolumn{1}{c|}{Ours}   & \multicolumn{1}{c|}{CMAS}   & DMAS   \\ \hline
1                  & \multicolumn{1}{c|}{$92.5\%$} & \multicolumn{1}{c|}{$37.2\%$} & $37.2\%$ & \multicolumn{1}{c|}{$70\%$}   & \multicolumn{1}{c|}{$28.5\%$} & $28.5\%$ & \multicolumn{1}{c|}{$52\%$} & \multicolumn{1}{c|}{$18\%$} & $18\%$ \\ \hline
3                  & \multicolumn{1}{c|}{$83.3\%$} & \multicolumn{1}{c|}{$30.3\%$} & $27.9\%$ & \multicolumn{1}{c|}{$55.6\%$} & \multicolumn{1}{c|}{$33\%$}   & $17\%$   & \multicolumn{1}{c|}{$20\%$} & \multicolumn{1}{c|}{$0\%$}  & $0\%$  \\ \hline
10                 & \multicolumn{1}{c|}{$95\%$}   & \multicolumn{1}{c|}{$60\%$}   & $5\%$    & \multicolumn{1}{c|}{$60\%$}   & \multicolumn{1}{c|}{$40\%$}   & $10\%$   & \multicolumn{1}{c|}{$0\%$}  & \multicolumn{1}{c|}{$0\%$}  & $0\%$  \\ \hline
15                 & \multicolumn{1}{c|}{$87.5\%$} & \multicolumn{1}{c|}{$25\%$}   & $0\%$    & \multicolumn{1}{c|}{$0\%$}    & \multicolumn{1}{c|}{$0\%$}    & $0\%$    & \multicolumn{1}{c|}{$0\%$}  & \multicolumn{1}{c|}{$0\%$}  & $0\%$  \\ \hline
\end{tabular}
\caption{\textcolor{black}{
Comparisons of planning accuracy of S-ATLAS (w/o help) against CMAS and DMAS, using GPT 3.5, Llama 2-7b, and Llama 3-8b.}}
\label{table:acc_no_help}
\end{table*}
\begin{appendices}
    \section{Comparisons Against Non-Conformalized LLM-based Planners}\label{app:Comp}

In this Appendix, we compare S-ATLAS against centralized and decentralized planners presented in \cite{chen2023scalable} which, however, do not employ CP and do not allow robots to ask for help. Our comparisons show that our proposed planner achieves better planning performance than these baselines planners (even when the help-request module is deactivated from our method). This performance gap becomes more apparent as the robot team size or the task complexity increases.

\subsection{Setting up Comparative Experiments}
\textbf{Baselines:} Particularly, we compare Alg. \ref{alg:MultiRobotLLM} against a recent decentralized (DMAS) and a centralized (CMAS) multi-robot planner proposed in \cite{chen2023scalable}. We select \cite{chen2023scalable} due to its token efficiency and scalability properties as shown in the empirical studies of that work. DMAS assigns an LLM agent to  each robot in the team. At every time step $t$, the LLM agents engage in dialogue rounds to pick their corresponding decisions $s_j(t)$. CMAS on the other hand considers a single LLM that is responsible for generating the multi-robot decision $\bbs(t)$. A key difference between S-ATLAS and these baselines is that the latter do not employ CP and do not allow robots to ask for help. Also, CMAS and DMAS do not employ the MCQA framework for action selection. 
To make comparisons fair we have applied all algorithms under the following settings: (i) All methods share the same prompt structure and select a decision from the same set $\ccalS$ constructed in Section \ref{sec:cpresult}. (ii) All methods terminate after a pre-defined mission horizon $H$. Cases where a planner fails to compute a plan within $H$ time steps are considered incorrect for all methods. (iii) We remove altogether the CP component from our planner since it does not exist in the baselines. The latter means that the robots never ask for help from human operators or their teammates; instead, S-ATLAS always picks the action $s_j(t)=\arg\max_{s\in\mathcal{S}}g(s|\ell_j(t))$ for all robots $j$ and time steps $t$. We emphasize that we have applied (iii) only to make comparisons fair against the baselines as otherwise our planner can outperform them by picking a low enough value for $\alpha$ and asking for help. This choice also allows us to assess the `nominal' performance of our planner when help from users is not available; evaluation of S-ATLAS without removing the CP-based help mode was presented in Sections \ref{sec:cpresult}-\ref{sec:knownoComp}.
%
%
%
%

\textbf{Evaluation Metrics}
Our main evaluation metrics include the planning accuracy, defined as the percentage of scenarios where a planner generates a feasible plan accomplishing the assigned task, length of plans, runtimes, and number of API calls; see also Remark \ref{rem:tokenEff2}.
To compute the planning accuracy, we manually check the correctness of the designed plans. 
\subsection{Evaluation}
We consider the $110$ scenarios $\xi$ used in Sec. \ref{sec:cpresult}. In what follows, we group these  scenarios based on the category (i.e., number of robots) they belong to and we report the performance of our planner and the baselines; see Table \ref{table:acc_no_help}.

\textbf{Case I:} We consider $27$ scenarios with single-robot tasks (i.e., $N=1$) defined over four sub-tasks and at most one  safety requirement. Each sub-task requires moving a specific semantic object to a desired location. 
The safety constraint requires the robot to avoid approaching or grabbing specific objects. For instance, a task may require the robot to move two tomatoes and one potato from the fridge to the kitchen counter (three sub-tasks), put the bread in the microwave (one sub-task) and to never grab the knife located on the kitchen counter (safety constraint). Other missions may offer to the robots multiple possible destinations for the objects.
We select as maximum horizon $H=15$ for all scenarios considered in this case study. \textcolor{black}{Notice that when $N=1$, CMAS and DMAS trivially share the same planning scheme.
Using GPT 3.5, the accuracy of our planner and CMAS/DMAS is $92.5\%$ and $37.2\%$, respectively. Using Llama 3-8b, the accuracy for our planner and CMAS/DMAS dropped to $70\%$ and $28.5\%$, respectively, while with Llama 2-7b, the accuracy for our planner and CMAS/DMAS further decreased to $52\%$ and $18\%$, respectively. This accuracy drop is expected as Llama 2-7b (7 billion parameters) and Llama 3-8b (8 billion parameters) are significantly smaller models compared to GPT-3.5 (175 billion parameters).}  
We remark that although $H=15$, the team manages to accomplish the task \textcolor{black}{in a shorter horizon defined as $\bar{H}$, such that $\bar{H}<H$. In this case, the action `stay idle' is selected by all robots for $t>\bar{H}$.}
The average horizon $\bar{H}$ for our planner, CMAS, and DMAS, across all successful scenarios, is $12$, $13$, and $13$, respectively, using GPT 3.5,  Llama 2-7b, and Llama 3-8b. 

\begin{table*}[t]
\centering
\begin{tabular}{|c|ccc|ccc|ccc|}
\hline
\multirow{2}{*}{N} & \multicolumn{3}{c|}{GPT 3.5}                                 & \multicolumn{3}{c|}{Llama 3-8b}                              & \multicolumn{3}{c|}{Llama 2-7b}                              \\ \cline{2-10} 
                   & \multicolumn{1}{c|}{Ours} & \multicolumn{1}{c|}{CMAS} & DMAS & \multicolumn{1}{c|}{Ours} & \multicolumn{1}{c|}{CMAS} & DMAS & \multicolumn{1}{c|}{Ours} & \multicolumn{1}{c|}{CMAS} & DMAS \\ \hline
1                  & \multicolumn{1}{c|}{0.15} & \multicolumn{1}{c|}{0.03} & 0.04 & \multicolumn{1}{c|}{3.8}  & \multicolumn{1}{c|}{1.0}  & 1.0  & \multicolumn{1}{c|}{2.2}  & \multicolumn{1}{c|}{0.4}  & 0.5  \\ \hline
3                  & \multicolumn{1}{c|}{0.32} & \multicolumn{1}{c|}{0.06} & 0.16 & \multicolumn{1}{c|}{4.46} & \multicolumn{1}{c|}{1.1}  & 0.9  & \multicolumn{1}{c|}{3.1}  & \multicolumn{1}{c|}{0.5}  & 0.9  \\ \hline
10                 & \multicolumn{1}{c|}{0.83} & \multicolumn{1}{c|}{0.17} & 0.24 & \multicolumn{1}{c|}{10.2}  & \multicolumn{1}{c|}{1.4}  & 1.5  & \multicolumn{1}{c|}{10}   & \multicolumn{1}{c|}{1.6}  & 1.7  \\ \hline
15                 & \multicolumn{1}{c|}{1.2}  & \multicolumn{1}{c|}{0.3}  & 0.5  & \multicolumn{1}{c|}{14.8} & \multicolumn{1}{c|}{1.5}  & 1.6  & \multicolumn{1}{c|}{14}   & \multicolumn{1}{c|}{1.5}  & 2    \\ \hline
\end{tabular}
\caption{\textcolor{black}{Comparisons of average runtimes (mins) required to generate the entire plan using GPT 3.5, Llama 2-7b, and Llama 3-8b.}}\vspace{-0.8cm}
\label{table:runtim_no_help}
\end{table*}

\textbf{Case II:} We consider $55$ scenarios with tasks $\phi$ defined over $N=3$ robots and four to eight sub-tasks and at most one safety constraint.  Notice that since the minimum number of sub-tasks in $\phi$ is larger than $N$, this means that at least one robot will have to eventually accomplish at least one sub-task. Recall from Section \ref{sec:problem} that these sub-tasks are not pre-assigned to the robots; instead, the LLM should compute both a feasible task assignment and a mission plan. The maximum horizon is $H=9$ in this case. \textcolor{black}{Using GPT 3.5, the accuracy of our planner, CMAS, and DMAS is $83.3\%$, $30.9\%$, and $27.3\%$, respectively. Using Llama 3-8b, the accuracy of our planner, CMAS, and DMAS dropped to $55.6\%$, $33\%$, and $17\%$, respectively. Using Llama 2-7b, the planning performance further decreased: our planner achieved $20\%$ accuracy, while both CMAS and DMAS failed to generate any correct plans.
%
} The average horizon $\bar{H}$ for all methods, across all successful scenarios, is $7$ using both GPT 3.5 and Llama 3-8b. Observe that unlike the baselines, the performance of our method has not dropped significantly compared to Case I (when coupled with GPT 3.5) despite the increase in the number of robots and the task complexity. 
%


\textbf{Case III:} Next, we consider scenarios over larger robot teams. Specifically, we consider $20$ scenarios associated with $N=10$ robots and missions $\phi$ involving four to eight sub-tasks and at most one safety constraint.  Notice here that the number of robots is larger than the maximum number of sub-tasks in $\phi$. This means that there should be robots with no tasks assigned by the LLM. The maximum horizon is $H=4$ in this case. The accuracy of our planner, CMAS, and DMAS is $95\%$, $60\%$, and $5\%$, respectively, \textcolor{black}{using GPT 3.5.} 
Despite the larger number of robots in this case study, compared to the previous one, our planner achieves high accuracy. Notably, the planning accuracy of our planner and CMAS has improved compared to Case Study II. We attribute this to the shorter horizon $\bar{H}$ required in this case study compared to the previous one. In fact, the average horizon $\bar{H}$ for all methods, across all successful scenarios, is $4$ time steps.  This shorter horizon is due to the fact that Case III involves a larger number of robots accomplishing the same number of sub-tasks as in Case II. \textcolor{black}{Using Llama 3-8b, our method, CMAS, and DMAS achieved $60\%$, $40\%$, and $10\%$ accuracy, respectively. All methods achieved $0\%$ accuracy when they are coupled with Llama 2-7b.}
%

\textbf{Case Study IV:} We consider $8$ scenarios that include $N=15$ robots and missions with ten sub-tasks and at most one safety constraint. We select $H=4$ as a maximum horizon. The accuracy of our planner, CMAS, and DMAS is $87.5\%$, $25\%$, and $0\%$, respectively. The average horizon $\bar{H}$ for all methods is $4$.
Observe that even though the horizons $\bar{H}$ in the previous and the current case study are similar, the performance of the baselines has dropped significantly possibly due to the larger number $N$ of robots as well as the larger number of associated sub-tasks. \textcolor{black}{All methods achieve $0\%$ accuracy when they are coupled with Llama 2-7b and Llama 3-8b.}

\begin{rem}[Computational Efficiency]\label{rem:tokenEff2}
\textcolor{black}{As demonstrated earlier, S-ATLAS achieves higher planning accuracy than the baselines even after deactivating the help mode. However, this increased accuracy comes at the cost of lower computational efficiency, indicated by a larger number of LLM queries, and increased runtimes for plan design; see Table \ref{table:runtim_no_help}. Specifically, while CMAS requires only one LLM query at each time $t$ to design a multi-robot action $\bbs(t)$ for all $N$ robots, our planner requires $N\cdot|\ccalS|$ queries (assuming $W=0$). On the other hand, DMAS requires at least $N$ LLM queries to design $\bbs(t)$. Thus, to design $\bbs(t)$ for scenarios from Case II, III, and IV, our method requires $84$, $280$, and $420$ LLM queries; DMAS requires on average $4$, $12$, and $15$ LLM queries, respectively. 
%
%
%
%
The runtimes of all case studies are demonstrated in Table \ref{table:runtim_no_help}. The total runtime to compute a plan depends on the total number of LLM queries, runtime per query  (determined by the number of tokens that the LLM has to process/generate), the horizon $H$, and the number $N$ of robots. 
\footnote{\textcolor{black}{Also, designing the local prediction set of robot $j$ at time $t$ requires obtaining the confidence scores for all possible $|\ccalS|$ decisions. In our implementation of S-ATLAS, obtaining these scores requires $|\ccalS|$ LLM queries which can occur in parallel since these queries are independent of one another. In our implementation, we query GPT 3.5 in parallel and the Llama models sequentially. Querying the Llama models in parallel was not possible due to excessive local computational requirements. This also resulted in longer runtimes for plan synthesis compared to when GPT 3.5 is employed as shown in Table \ref{table:runtim_no_help}. Note that Llama models are run on our local machines, whereas GPT-3.5 is accessed via OpenAI’s servers.}} 
Finally, observe that CMAS and DMAS have comparable runtimes despite their different number of LLM queries. This is because each query for CMAS takes less time than for DMAS. 
%
%
%
}
\end{rem}

\textbf{Summary of Comparisons:} The above empirical results show that the performance gap, in terms of planning accuracy, between S-ATLAS and the baselines increases significantly as the scenario complexity increases, where the latter is determined primarily by the horizon $\bar{H}$, the robot team size $N$, and the number of sub-tasks in $\phi$. Specifically, as $\bar{H}$ increases (e.g., due to a low ratio of the number of robots over the number of sub-tasks), the performance of the baselines drops; see e.g., Case I and II. Also, for scenarios associated with similar values for $\bar{H}$, the performance of the baselines tends to drop as the number of robots and sub-tasks increase; see Case III-IV. 
The accuracy of S-ATLAS seems to be more robust to such variations. We attribute this empirical performance improvement to the MCQA framework, as it attempts to eliminate hallucinations. To the contrary, CMAS and DMAS require the employed LLMs to generate new tokens to design multi-robot plans increasing the risk of hallucinations. Notably, CMAS achieves better planning accuracy than DMAS which is also consistent with the results presented in \cite{chen2023scalable}. \textcolor{black}{We note again that Alg. \ref{alg:MultiRobotLLM} can achieve desired mission success rates (see Section \ref{sec:cpresult}) when coupled with CP, a capability that is missing from existing LLM-based multi-robot planners including the inspiring work in \cite{chen2023scalable}; see Sections \ref{sec:cpresult}-\ref{sec:knownoComp}.}
    %

\begin{rem}[Hybrid Planners]
    The work in \cite{chen2023scalable} proposes also hybrid frameworks as extensions of CMAS and DMAS. Specifically, once CMAS generates a plan, LLMs are employed to detect and resolve conflicts in a centralized or decentralized way. These conflicts include robot collisions that may occur during the plan execution, as well as the assignment of actions that do not belong to the action set. Our planner prevents the latter from happening due to MCQA framework. Although our proposed algorithm can also be integrated with similar frameworks to prevent collisions during the plan execution, this aspect is beyond the scope of this work. 
\end{rem}

\section{\textcolor{black}{Conformal Prediction \\ with Multiple Acceptable Decisions}}\label{app:proof}

Our conformal prediction analysis to design prediction sets, provided in Section \ref{sec:cp}, as well as our mission completion rate guarantees (Theorem \ref{thm1}) can be extended to cases where multiple feasible solutions exist. As discussed in Remark \ref{rem:multipleFeas1}, the key difference lies in the construction of the calibration dataset. For each calibration mission scenario, among all feasible plans, we select the one constructed by picking at each planning step the decision with the highest LLM confidence score (among all other feasible decisions that contribute to mission progress) while also following an ordered set $\ccalI$. CP can then be applied as usual, yielding prediction sets that contain, with user-specified probability, the plan consisting of the feasible decisions with the highest LLM confidence scores among all feasible plans. In what follows, we provide a more formal description of this approach, which is adapted from Appendices A3-A4 of \cite{ren2023robots}. We first present the analysis for single-robot, single-step plans and then generalize it to multi-robot, multi-step plans.

\underline{Single-robot \& Single-step plan}: Initially, we focus on scenarios with $N=1$ and $H=1$. 
%
First, we sample $M$ independent scenarios from $\ccalD$. We refer to these scenarios as calibration scenarios. For each calibration scenario $i\in\{1,\dots,M\}$, we construct its equivalent prompt $\ell_{j,\text{calib}}^i$ associated with (the single) robot $j$. For each prompt, we assume a set of $R_i\geq 1$ feasible plans collected in the set $\mathcal{T}^i_{j,\text{calib}} = \{\tau^i_{r, \text{calib}} =s_{j,r,\text{calib}}^{i} \}_{r=1}^{R_i}$. This way we construct a calibration dataset $\ccalM=\{(\ell_{j,\text{calib}}^i,\mathcal{T}_{j,\text{calib}}^{i})\}_{i=1}^M$. Consider a function $F$ which, given a prompt $\ell$ and a set $\mathcal{T}$ of feasible decisions, returns a single feasible decision $s\in\ccalT$, i.e., $F(\ell, \mathcal{T}) = s$. In our implementation, we define $F$ so that, given $\ell_{j,\text{calib}}^i$, it  returns the decision $s\in\mathcal{T}^i_{j,\text{calib}}$, among all other ones in $\mathcal{T}^i_{j,\text{calib}} $, that has the largest LLM confidence score $g(s|\ell_{j,\text{calib}}^i)$.

%
We apply $F$ in the support of our distribution $D$ inducing a distribution $\ccalD'$. The induced calibration dataset is also defined as $\ccalM'=\{(\ell_{j,\text{calib}}^i,\hat{s}_{j,\text{calib}}^{i})\}_{i=1}^M$, where $\hat{s}_{j,\text{calib}}^{i} = F(\ell_{j,\text{calib}}^i, \mathcal{T}_{j,\text{calib}}^{i})$. Then we apply conformal prediction as usual using the calibration dataset $\mathcal{M}'$ by computing the quantile $q=\frac{(M+1)(1-\alpha)}{M}$. 

Consider an unseen validation/test scenario drawn from $\ccalD$. We convert this scenario into its equivalent prompt $\ell_{j,\text{test}}$. CP then generates a prediction set $\mathcal{C}(\ell_{j,\text{test}})$ using $q$:
\begin{equation}\label{eq:cov1}
\mathcal{C}(\ell_{j,\text{test}}) = \{ s_j | g(s_j | \ell_{j,\text{test}}) > 1-q \},
\end{equation}
This prediction set is supported by the following marginal guarantees:
\begin{equation}
P(\hat{s}_{j,\text{test}}\in \mathcal{C}(\ell_{j,\text{test}})) = P(F(\ell_{j,\text{test}}, \mathcal{T}_{j, \text{test}}) \in \mathcal{C}(\ell_{j,\text{test}})) \geq 1-\alpha,
\end{equation}
for some $\alpha\in(0,1)$, only if the test mission scenario $(\ell_{j,\text{test}},\ccalT_{j,\text{test}})$ sampled from $\ccalD$ is i.i.d., with the calibration data sampled from $\ccalD'$. This is true
since functions of independent random variables are independent, and functions of identically distributed
random variables are identically distributed if the functions are measurable. Since $F$ is measurable, the coverage guarantee in \eqref{eq:cov1} holds. 

\underline{Multi-robot \& Multi-step plan:} Next, we consider the case where $N\geq 1$ and $H\geq 1$. We sample $M\geq 1$ independent calibration scenarios from $\mathcal{D}$. Each scenario will have a sequence of $T_i=H_i \cdot N_i \geq 1$ prompts denoted by 
\begin{align}\label{eq:seqProm}
&\bar{\ell}_{\text{calib}}^i=\underbrace{\ell_{1,\text{calib}}^i(1),\dots,\ell_{j,\text{calib}}^i(1),\dots,\ell_{N_i,\text{calib}}^i(1)}_{t=1},\dots,\nonumber\\&
\underbrace{\ell_{1,\text{calib}}^i(t'),\dots,\ell_{j,\text{calib}}^i(t'),\dots,\ell_{N_i,\text{calib}}^i(t')}_{t=t'},\dots,\nonumber\\&
\underbrace{\ell_{1,\text{calib}}^i(H_i),\dots,\ell_{j,\text{calib}}^i(H_i),\dots,\ell_{N_i,\text{calib}}^i(H_i)}_{t=H_i}.
\end{align}

And each sequence will have a set of sequences of true labels $\bar{\mathcal{T}}_{\text{calib}}^i = \{\tau^{i,r}_{\text{calib}} \}_{r=1}^{R_i}$, where

\begin{align}\label{eq:Plans3}
&\tau_{r,\text{calib}}^{i}=\underbrace{s^i_{1,r,\text{calib}}(1),\dots,s^i_{j,r,\text{calib}}(1),\dots,s^i_{N_i,r,\text{calib}}(1)}_{=\bbs^r_{\text{calib}}(1)},\dots,\nonumber\\&
\underbrace{s^i_{1,r,\text{calib}}(t'),\dots,s^i_{j,r,\text{calib}}(t'),\dots,s^i_{N_i,r,\text{calib}}(t')}_{=\bbs^i_{r,\text{calib}}(t')},\dots,\nonumber\\&
\underbrace{s^i_{1,r,\text{calib}}(H_i),\dots,s^i_{j,r,\text{calib}}(H_i),\dots,s^i_{N_i,r,\text{calib}}(H_i)}_{=\bbs^i_{r,\text{calib}}(H_i)}.
\end{align}
representing a correct plan that can satisfy the task. This gives rise to a calibration dataset $\mathcal{M} = \{(\bar{\ell}_{\text{calib}}^i, \bar{\mathcal{T}}_{\text{calib}}^i) \}_{i=1}^M$. Here for simplicity, we denote by $\bar{\ell}_{\text{calib}}^i(k)$ and $\tau_{r,\text{calib}}^{i}(k)$ the $k$-th entry in $\bar{\ell}_{\text{calib}}^i$ and $\tau_{r,\text{calib}}^{i}$, respectively, where $k\in\{1,\dots,T_i\}$. 

Note that unlike the `single-robot \& single-step plan' here we cannot apply $F$ to the set of true plans $\bar{\mathcal{T}}_{\text{calib}}^i$ in each step since the correct action at time step $k$ depends on the sequence of previously chosen correct actions. 

Let $\bar{\mathcal{T}}^i_{k, \text{calib}}$ denote the set of feasible/correct options at time step $k$ given an initial prompt $\bar{\ell}_{\text{calib}}^i(1)$ and the sequence  of previously chosen feasible actions until step $k-1$. 
%
Then, we auto-regressively define the following sequence:

\begin{equation}
\bar{F}_1(\bar{\ell}_{\text{calib}}^i, \bar{\mathcal{T}}_{\text{calib}}^i)= F(\bar{\ell}_{\text{calib}}^i(1), \bar{\mathcal{T}}_{1,\text{calib}}^i) 
%
\end{equation}
%
and
\begin{equation}
\bar{F}_{k}(\bar{\ell}_{\text{calib}}^i, \bar{\mathcal{T}}_{\text{calib}}^i)= \bar{F}_{k-1} (\bar{\ell}_{\text{calib}}^i, \bar{\mathcal{T}}_{\text{calib}}^i) \bigcup F (\bar{\ell}_{\text{calib}}^i, \bar{\mathcal{T}}_{k,\text{calib}}^i),
\end{equation}
for all $k\in\{1,\dots,T_i\}$. Intuitively, $\bar{\mathcal{T}}_{\text{calib}}^i$ collects all possible plans that can be designed by selecting a feasible action at each step $k$. On the other hand, $\bar{F}_{k}(\bar{\ell}_{\text{calib}}^i, \bar{\mathcal{T}}_{\text{calib}}^i)$ determines the plan constructed by always selecting the action with the highest LLM confidence score among all feasible actions.

By applying $\bar{F}_{T_i}$ to the sequence $\ccalD$ and the calibration dataset, we get the induced distribution $\ccalD'$ and the dataset $\ccalM'= \{ (\bar{\ell}^i_{\text{calib}}, \hat{\tau}^i_{\text{calib}}) \}_{i=1}^M$, where $\hat{\tau}^i_{\text{calib}}=\bar{F}_{T_i}(\bar{\ell}_{\text{calib}}^i, \bar{\mathcal{T}}_{\text{calib}}^i)$. Then, we apply conformal prediction as usual by computing the quantile $\bar{q}$ using the calibration dataset $\ccalM'$. 

Consider a new test scenario $(\bar{\ell}_{\text{test}},\bar{\mathcal{T}}_{\text{test}})$ drawn from $\mathcal{D}$ that is defined over $N_{\text{test}}$ robots and horizon $H_{\text{test}}$, where $\bar{\ell}_{\text{test}}$ corresponds to a sequence of prompts
$\bar{\ell}_{\text{test}}=\bar{\ell}_{\text{test}}(1),\dots,\bar{\ell}_{\text{test}}(k),\dots,\bar{\ell}_{\text{test}}(T_{\text{test}})$,
with $T_{\text{test}}=H_\text{test}\cdot N_{\text{test}}$, and $\bar{\mathcal{T}}_{\text{test}}$ is the set of all feasible plans. \textcolor{black}{For the test sequence $\bar{\ell}_{\text{test}}$, we obtain $\hat{\tau}_{\text{test}} = \bar{F}_{T_{\text{test}}}(\bar{\ell}_{\text{test}}, \bar{\mathcal{T}}_{\text{test}})$.}

Using the calibration dataset $\ccalM'$ and the corresponding quantile $\bar{q}$, we construct the following prediction set:
\begin{equation}\label{eq:pred3}
\bar{\mathcal{C}}(\bar{\ell}_{\text{test}})=\{\tau~|~\bar{g}(\tau|\bar{\ell}_{\text{test}})>1-\bar{q}\},
\end{equation}

This set satisfies the following coverage guarantee,
\begin{equation}\label{eq:cov2}
P(\bar{F}_{T_{\text{test}}}(\bar{\ell}_{\text{test}}, \bar{\mathcal{T}}_{\text{test}}) \in \bar{\mathcal{C}}(\bar{\ell}_{\text{test}})) \geq 1-\alpha,
\end{equation}
only if  $(\bar{\ell}_{\text{test}},\bar{\mathcal{T}}_{\text{test}})$  sampled from $\ccalD$ is i.i.d., with the calibration data $\ccalM'$ sampled from $\ccalD'$; in words, \eqref{eq:cov2} means that the set $\bar{\mathcal{C}}(\bar{\ell}_{\text{test}})$ contains the plan $\bar{F}_{T_{\text{test}}}(\bar{\ell}_{\text{test}}, \bar{\mathcal{T}}_{\text{test}})$ with probability at least equal to $1-\alpha$. This i.i.d. requirement is satisfied using the same argument used for single-robot \& single-step plans. We note that the prediction set in \eqref{eq:pred3} can be constructed causally using exactly the same steps presented in Section \ref{sec:cp} as that analysis does not rely uniqueness of feasible solutions.
\end{appendices}

\end{document}